\documentclass[11pt,twoside]{article}

\usepackage{enumitem}

\usepackage[utf8]{inputenc} 
\usepackage[T1]{fontenc}    
\usepackage{booktabs}       
\usepackage{amsfonts}       
\usepackage{nicefrac}       
\usepackage{microtype}      

\usepackage{subfigure}
\usepackage{epsf}
\usepackage{epsfig}
\usepackage{fancyhdr}
\usepackage{graphics}
\usepackage{graphicx}
\usepackage{psfrag}
\usepackage{fullpage}
\usepackage{pdfpages}

\usepackage{url}
\usepackage[colorlinks,linkcolor=magenta,citecolor=blue, pagebackref=true,backref=true]{hyperref}
\renewcommand*{\backrefalt}[4]{%
    \ifcase #1 \footnotesize{(Not cited.)}%
    \or        \footnotesize{(Cited on page~#2.)}%
    \else      \footnotesize{(Cited on pages~#2.)}%
    \fi}
    
\usepackage{amsthm}
\usepackage{amsmath}
\usepackage{amssymb,bbm}
\usepackage{caption}
\usepackage{algorithmic}
\usepackage{algorithm}
\usepackage{multicol}
\usepackage{caption}
\usepackage{color}

\usepackage{changes}

\usepackage{enumitem}

\usepackage[utf8]{inputenc} 
\usepackage[T1]{fontenc}    
\usepackage{booktabs}       
\usepackage{amsfonts}       
\usepackage{nicefrac}       
\usepackage{microtype}      

\usepackage{subfigure}
\usepackage{epsf}
\usepackage{epsfig}
\usepackage{fancyhdr}
\usepackage{graphics}
\usepackage{graphicx}
\usepackage{psfrag}
\usepackage{pdfpages}

\usepackage{url}
\usepackage[colorlinks,linkcolor=magenta,citecolor=blue, pagebackref=true,backref=true]{hyperref}

\usepackage{amsthm}
\usepackage{amsmath}
\usepackage{amssymb,bbm}
\usepackage{caption}
\usepackage{algorithmic}
\usepackage{multicol}
\usepackage{caption}
\usepackage{color}

\usepackage{changes}
\definechangesauthor[name=Huy Nguyen, color=purple]{HN}

\DeclareMathOperator{\Cov}{Cov}
\DeclareMathOperator{\Var}{Var}

\newtheorem{assumption}{Assumption}
\newtheorem{remark}{Remark}

\newtheorem{lemma}{Lemma}
\newtheorem{theorem}{Theorem}
\newtheorem{proposition}{Proposition}
\newtheorem{definition}{Definition}
\newtheorem{corollary}{Corollary}

\def\II{\mathbf{I}}

\def\EE{\mathbb{E}}

\def\bx{\boldsymbol{x}}
\def\by{\boldsymbol{y}}
\def\bX{\boldsymbol{X}}
\def\bY{\boldsymbol{Y}}
\def\bU{\boldsymbol{U}}
\def\bV{\boldsymbol{V}}

\def\bTheta{\boldsymbol{\Theta}}
\def\balpha{\boldsymbol{\alpha}}

\newcommand{\br}{\mathbb{R}}
\newcommand{\bigO}{\mathcal{O}}

\usepackage[utf8]{inputenc} 
\usepackage[T1]{fontenc}    
\usepackage{booktabs}       
\usepackage{amsfonts}       
\usepackage{nicefrac}       
\usepackage{microtype}      

\usepackage{subfigure}
\usepackage{epsf}
\usepackage{epsfig}
\usepackage{fancyhdr}
\usepackage{graphics}
\usepackage{graphicx}
\usepackage{psfrag}
\usepackage{fullpage}
\usepackage{pdfpages}

\usepackage{url}
\usepackage[colorlinks,linkcolor=magenta,citecolor=blue, pagebackref=true]{hyperref}
\renewcommand*{\backrefalt}[4]{%
    \ifcase #1 \footnotesize{(Not cited.)}%
    \or        \footnotesize{(Cited on page~#2.)}%
    \else      \footnotesize{(Cited on pages~#2.)}%
    \fi}

\usepackage{color}

\usepackage{amsthm}
\usepackage{amsmath}
\usepackage{amssymb,bbm}
\usepackage{caption}
\usepackage{algorithmic}
\usepackage{algorithm}
\usepackage{textcomp}
\usepackage{siunitx}
\usepackage{wrapfig}
\usepackage{algorithmic}
\usepackage{algorithm}
\usepackage{multirow}
\usepackage{multicol}
\newcommand{\widgraph}[2]{\includegraphics[keepaspectratio,width=#1]{#2}}

\begin{document}
\begin{center}

{\bf{\LARGE{Fast Approximation of the\\\vspace{0.2cm} Generalized Sliced-Wasserstein Distance}}}

\vspace*{.2in}
 {\large{
 \begin{tabular}{ccccc}
 Dung Le$^{\star,\diamond}$& Huy Nguyen$^{\star,\dagger}$& Khai Nguyen$^{\star,\dagger}$& Trang Nguyen$^{\mathsection}$& Nhat Ho$^{\dagger}$
 \end{tabular}
}}

\vspace*{.2in}

\begin{tabular}{c}
\'{E}cole Polytechnique$^{\diamond}$; University of Texas, Austin$^{\dagger}$; \\
Hanoi University of Science and Technology$^{\mathsection}$
\end{tabular}
\vspace*{.2in}

\today

\vspace*{.2in}

\begin{abstract}
Generalized sliced Wasserstein distance is a variant of sliced Wasserstein distance that exploits the power of non-linear projection through a given defining function to better capture the complex structures of the probability distributions. Similar to sliced Wasserstein distance, generalized sliced Wasserstein is defined as an expectation over random projections which can be approximated by the Monte Carlo method. However, the complexity of that approximation can be expensive in high-dimensional settings. To that end, we propose to form deterministic and fast approximations of the generalized sliced Wasserstein distance by using the concentration of random projections when the defining functions are polynomial function, circular function, and neural network type function. Our approximations hinge upon an important result that one-dimensional projections of a high-dimensional random vector are approximately Gaussian.
\end{abstract}

\let\thefootnote\relax\footnotetext{$\star$ Dung Le, Huy Nguyen and Khai Nguyen contributed equally to this work.}
\end{center}

\section{Introduction}
\label{sec:introduction}
Sliced Wasserstein (SW)~\cite{bonneel2015sliced} distance has become a core member in the family of probability metrics that are based on optimal transport~\cite{Villani-09}. Compared to Wasserstein distance, SW provides a lower computational cost thanks to the closed-form solution of optimal transport in one dimension. In particular, when dealing with probability measures with at most $n$ supports, the computational complexity of SW is $\mathcal{O}(n\log n)$ while that of Wasserstein distance is $\mathcal{O}(n^3 \log n)$~\cite{pele2009} when being solved via the interior point methods or $\mathcal{O}(n^2)$~\cite{altschuler2017near, lin2019efficient, Lin-2022-Efficiency} when being approximated by its entropic regularized version. Furthermore, the memory complexity of SW is only $\mathcal{O}(n)$ in comparison with $\mathcal{O}(n^2)$ of Wasserstein distance (due to the storage of a cost matrix). Additionally, the statistical estimation rate (or the sample complexity) of SW does not depend on the dimension (denoted as $d$) like Wasserstein distance. In particular, the sample complexity of the former is $\mathcal{O}(n^{-1/2})$~\cite{Bobkov_2019}, whereas it is $\mathcal{O}(n^{-1/d})$~\cite{Fournier_2015} for the 
Wasserstein distance. Therefore, the SW does not suffer from the curse of dimensionality.

\vspace{0.5em}
\noindent
Due to the practicality of SW, several improvements and variants of that distance have been explored recently in the literature. For instance, selective discriminative projecting directions techniques are proposed in~\cite{deshpande2019max,nguyen2021distributional,nguyen2021improving}; a SW variant that augments original measures to higher dimensions for better linear separation is introduced in~\cite{chen2022augmented}; a SW variant on the sphere is defined in~\cite{bonet2022spherical}; a SW variant that uses convolution slicer for projecting images is proposed in~\cite{nguyen2022revisiting}. However, the prevailing trend of the current works on SW is focused on its application. Indeed, SW is used in generative modeling~\cite{deshpande2018generative,nguyen2022amortized,dai2021sliced}, domain adaptation~\cite{wu2019sliced}, clustering~\cite{kolouri2018slicedgmm}, and Bayesian inference~\cite{nadjahi2020approximate,yi2021sliced}.

\vspace{0.5em}
\noindent
To further enhance the ability of SW, Kolouri et al.~\cite{kolouri2019generalized} propose using non-linear projecting defining functions for SW instead of conventional linear projecting. This extension leads to a generalized class of sliced probability distances, named the generalized sliced Wasserstein (GSW) distance. Despite being more expressive, GSW also needs to be approximated by the Monte Carlo method as SW. In greater detail,
the definition of GSW is an expectation over random projections via certain defining functions of Wasserstein distance between corresponding one-dimensional projected probability measures. In general, the expectation is intractable to compute; hence, Monte Carlo samples are used to approximate the expectation as mentioned. It is shown in both theory and practice that the number of Monte Carlo samples (the number of projections) should be large for good performance and approximation of sliced probability metrics~\cite{nadjahi2020statistical,deshpande2019max}.

\vspace{0.5em}
\noindent
\textbf{Contribution.} In this work, we aim to overcome the projection complexity of the GSW by deriving fast approximations of that distance that do not require using Monte Carlo random projecting directions. We follow the approach of deterministic approximation of the SW in~\cite{nadjahi2021fast}. The key factor in our fast approximations of the GSW is the Gaussian concentration of the distribution of low-dimensional projections of high-dimensional random variables~\cite{sudakov1978typical,diaconis1984asymptotics}. Our results cover the settings when the (non-linear) defining functions are polynomial function with odd degree, circular function, and neural network type, which had been discussed and utilized in~\cite{kolouri2019generalized}. 

\vspace{0.5em}
\noindent
\textbf{Organization.} The paper is organized as follows. We provide background on Wasserstein distance, sliced Wasserstein distance and its fast approximation, as well as the generalized sliced Wasserstein distance in Section~\ref{sec:background}. We then study the fast approximation of generalized sliced Wasserstein distance when the defining function is polynomial with odd degree in Section~\ref{sec:polynomial} and when the defining function is neural network type in Section~\ref{sec:neural_net}. The discussion with an approximation when the defining function is circular is in Appendix~\ref{sec:circular}. Finally, we give experiment results for the approximation error of the proposed approximate generalized sliced Wasserstein distance in Section~\ref{sec:experiments} and conclude the paper in Section~\ref{sec:conclusion}. The remaining proofs of the key results in the paper are deferred to Appendix~\ref{appendix:missing_proofs}.

\vspace{0.5em}
\noindent
\textbf{Notation.} We use the following notations throughout our paper. Firstly, we denote by $\mathbb{N}$ the set of all positive integers. For any $d\in\mathbb{N}$ and $p\in\mathbb{N}$, $\mathcal{P}_p(\br^d)$ stands for the set of all probability measures in $\br^d$ with finite moments of order $p$ whereas $\mathbb{S}^{d-1}:=\{\theta\in\br^d:\|\theta\|=1\}$ denotes the $d$-dimensional unit sphere where $\|\cdot\|$ is the Euclidean norm. Additionally, $\gamma_d$ represents the Gaussian distribution in $\br^d$, $\mathcal{N}(0,d^{-1}\II_d)$ in which $\II_d$ is an identity matrix of size $d\times d$. Meanwhile, we denote  $L^1(\br^d):=\{f:\br^d\to\br:\int_{\br^d}|f(x)|dx<\infty\}$ as the set of all absolutely integrable functions on $\br^d$. Next, for any set $A$, we denote by $|A|$ its cardinality. Finally, for any two sequences $(a_n)$ and $(b_n)$, the notation $a_n=\bigO(b_n)$ indicates that $a_n\leq Cb_n$ for all $n\in\mathbb{N}$ where $C$ is some universal constant.
\section{Backgrounds}
\label{sec:background}
In this section, we first revisit Wasserstein distance and the conditional central limit theorem for Gaussian projections. We then present background on sliced Wasserstein distance and its fast approximation. Finally, we recall the definition of generalized sliced Wasserstein distance, which is focused on in this paper.
\subsection{Wasserstein Distance}
\label{sec:wasserstein_distance}
Let $p\geq 1$ and $\mu,\nu$ be two probability measures on $\br^d$, $d\geq 1$, with finite moments of order $p$. Then, the $p$-Wasserstein distance between $\mu$ and $\nu$ is defined as follows:
\begin{align*}
    W_p(\mu,\nu):=\left(\inf_{\pi\in\Pi(\mu,\nu)}\int_{\br^d\times\br^d}\|x-y\|^p d\pi(x,y)\right)^{\frac{1}{p}},
\end{align*}
where $\|\cdot\|$ denotes the Euclidean norms, and $\Pi(\mu,\nu)$ is the set of all probability measures on $\br^d\times \br^d$ which admit $\mu$ and $\nu$ as their marginals with respect to the first and second variables.

\vspace{0.5em}
\noindent
Next, we review an important result about the concentration of measure phenomenon, which states that under mild assumptions, one-dimensional projections of a high-dimensional random vector are approximately Gaussian. Specifically, we have the following theorem.
\begin{theorem}[\cite{kolouri2019generalized} Theorem 1]
\label{theorem:clt_random_projections}
For any $d\geq 1$, let $\mu$ denote the distribution of $X_{1:d}=(X_1,\ldots,X_d)$ and $\gamma_d=\mathcal{N}(0,d^{-1}\mathbf{I}_d)$ be a Gaussian distribution. Assume that $\mu\in\mathcal{P}_2(\br^d)$, then there exists a universal constant $C\geq 0$ such that: 
\begin{align*}
    \int_{\mathbb{R}^d} {W}_2^2\left(\theta_{\sharp}^{*}\mu, N(0,d^{-1}\mathfrak{m}_2(\mu))\right)d \gamma_d(\theta) \leq C~\Xi_{d}(\mu), 
\end{align*}
where $\theta^*:\br^d\to\br$ denotes the linear form $x\mapsto\langle \theta,x\rangle$, $\theta^*_{\sharp}\mu$ indicates the push-forward measure of $\mu$ by $\theta^*$ and
\begin{align}
    \Xi_d(\mu) &= d^{-1}\Big\{A(\mu) + [\mathfrak{m}_2(\mu)B_1(\mu)]^{1/2}+ \mathfrak{m}_2(\mu)^{1/5}B_2(\mu)^{4/5}\Big\}, \nonumber\\
    \label{eq:xi_definition}
    \mathfrak{m}_2(\mu) & = \mathbb{E}\left[\lVert X_{1:d}\rVert^2\right],\\
    A(\mu) &= \mathbb{E}\left[\left|\lVert X_{1:d}\rVert^2 - \mathfrak{m}_2(\mu)\right|\right],\nonumber\\
    B_k(\mu) &= \mathbb{E}^{1/k}\left[|\langle X_{1:d}, X'_{1:d}\rangle|^{k}\right],\nonumber
\end{align}
with $k\in\{1,2\}$ and $X^{\prime}_{1:d}$ is an independent copy of $X_{1:d}$.
\end{theorem}
\vspace{0.5em}
\noindent
It is worth noting that the above result only holds for the $2$-Wasserstein distance.
\subsection{Sliced-Wasserstein Distance And Its Fast Approximation}
To adapt the result of Theorem~\ref{theorem:clt_random_projections} to the sliced-Wasserstein setting, Nadjahi et al.~\cite{nadjahi2021fast} introduce a new version of SW distance in which projections are sampled from the Gaussian distribution rather than on the sphere as usual. In particular,
\paragraph{Sliced-Wasserstein Distance:} Let $p\geq 1$ and a Gaussian measure $\gamma_d=\mathcal{N}(0,d^{-1}\mathbf{I}_d)$ where $d\geq 1$. Then, the sliced Wasserstein distance of order $p$ with Gaussian projections between two probability measures $\mu\in\mathcal{P}_p(\br^d)$ and $\nu\in\mathcal{P}_p(\br^d)$ is defined as follows:
\begin{align}
    SW_p(\mu,\nu):=\left(\int_{\br^d} W^p_p(\theta^*_{\sharp}\mu,\theta^*_{\sharp}\nu)d\gamma_d(\theta)\right)^{\frac{1}{p}}.
\end{align}

\vspace{0.5em}
\noindent
The notation $\theta^* \sharp \mu$ is equivalent to the Radon Transform of $\mu$ given the projecting direction $\theta^*$~\cite{kolouri2019generalized}. By leveraging Theorem~\ref{theorem:clt_random_projections}, Nadjahi et al.~\cite{nadjahi2021fast} provide the following bound for the sliced-Wasserstein distance between any two probability measures with finite second moments.
\begin{proposition}[\cite{nadjahi2021fast}, Theorem 1]
\label{propo:approximate_sliced_by_wasserstein}
Let $\mu,\nu \in \mathcal{P}_2(\mathbb{R}^d)$ be two probability measures in $\br^d$. Consider two Gaussian distributions $\eta_\mu=\mathcal{N}(0, d^{-1}\mathfrak{m}_2(\mu))$ and $\eta_\nu=\mathcal{N}(0, d^{-1}\mathfrak{m}_2(\nu))$, where $\mathfrak{m}_2(\mu),\mathfrak{m}_2(\nu)$ are given in equation~\eqref{eq:xi_definition}. Then, there exists a universal constant $C > 0$ such that:
\begin{equation}
\label{eqn:approximate_sliced_by_wasserstein}
    |SW_2(\mu, \nu) - W_2(\eta_\mu,\eta_\nu)| \leq C(\Xi_d(\mu) + \Xi_d(\nu))^{\frac{1}{2}},
\end{equation}
where $\Xi_d(\mu)$ and $\Xi_d(\nu)$ are defined in equation~\eqref{eq:xi_definition}.
\end{proposition}
\vspace{0.5em}
\noindent
Note that equation~\eqref{eqn:approximate_sliced_by_wasserstein} can be simplified by using the closed-form expression of Wasserstein distance between two Gaussians distributions $\eta_{\mu}$ and $\eta_{\nu}$, which is given by
\begin{align*}
    W_2(\eta_\mu,\eta_\nu)=d^{-\frac{1}{2}}\Big |\sqrt{\mathfrak{m}_2(\mu)}-\sqrt{\mathfrak{m}_2(\nu)}\Big|.
\end{align*}
According to \cite{nadjahi2021fast}, $\Xi_d(\mu)$ and $\Xi_d(\nu)$ cannot be shown to converge to 0 if the data are not centered. Fortunately, they demonstrate that there is a relation between $SW_2(\mu,\nu)$ and $SW_2(\bar{\mu},\bar{\nu})$, where $\bar{\mu}$ and $\bar{\nu}$ are centered versions of $\mu$ and $\nu$, respectively.
\begin{proposition}
\label{propo:transform_to_the_centered}
Let $\mu, \nu \in \mathcal{P}_2(\mathbb{R}^d)$ be two probability measures in $\br^d$ with respective means $m_{\mu}$ and $m_{\nu}$. Then, the Sliced-Wasserstein distance of order 2 between $\mu$ and $\nu$ can be decomposed as:
\begin{align}
\label{eqn:transform_to_the_centered}
    SW_2^2(\mu, \nu) = SW_2^2(\bar{\mu},\bar{\nu}) + d^{-1}\lVert m_{\mu} - m_{\nu}\rVert^2.
\end{align}
\end{proposition}
\vspace{0.5em}
\noindent
As a consequence, Nadjahi et al.~\cite{nadjahi2021fast} successfully derive a deterministic approximation for $SW_2(\mu,\nu)$ as follows:
\begin{align}
\label{eqn:the_main_formula_for_approximate}
    &\widehat{SW}_2^2(\mu, \nu) = W_2^2(\eta_{\bar{\mu}},\eta_{\bar{\nu}}) + d^{-1}\lVert {m}_{\mu} - {m}_{\nu}\rVert^2.
\end{align}
\subsection{Generalized Sliced-Wasserstein Distance}
Inspired by the approximation of SW distance in equation~\eqref{eqn:the_main_formula_for_approximate}, we manage to extend that result to the setting of Generalized Sliced-Wasserstein (GSW) distance in this work. Before exploring the aforementioned extension, it is necessary to recall the definition of GSW distance.
\paragraph{Generalized Sliced-Wasserstein Distance:} Let $g$ be a defining function \cite{nadjahi2021fast} and $\delta$ be the Dirac delta function, then the generalized Radon transform (GRT) of an integrable function $I\in L^1(\br^d)$, denoted by $\mathcal{G}I$, is defined as follows:
\begin{align}
    \mathcal{G}I(t,\theta):=\int_{\br^d}I(x)\delta(t-g(x,\theta))dx. \label{eq:generalized_Radon}
\end{align}
When $g(x,\theta) = \langle x,\theta \rangle$, GRT reverts into the conventional Radon Transform which is used in SW distance. By using the GRT, the GSW distance is given by:
\begin{align*}
    {GSW}_p({\mu},{\nu}) := \left(\int_{\br^d}W_p^p\left(\mathcal{G}I_{\mu}(\cdot,\theta),\mathcal{G}I_{\nu}(\cdot,\theta)\right)d\gamma_d(\theta)\right)^{\frac{1}{p}},
\end{align*}
where $I_\mu,I_\nu\in L^1(\br^d)$ are probability density functions of measures $\mu$ and $\nu$, respectively. Here, with a slight abuse of notation, we use $W_p(\mu, \nu)$  and $W_p(I_\mu,I_\nu)$ interchangeably. In this paper, we will also use the pushforward measures notation to define GSW e.g., $g^\theta\sharp \mu$ denotes the GRT of $\mu$ given the defining function $g$ and its parameter $\theta$.

\vspace{0.5em}
\noindent
In order for the GSW distance to become a proper metric, the GRT must be essentially an injective function. There is a line of work \cite{andrew2017grt,beylkin1984inversion} studying the sufficient and necessary conditions for the injectivity of GRT, which finds that the GRT is injective when $g$ is either a polynomial defining function or a circular defining function. By contrast, it is non-trivial to show that GRT is injective when $g$ is a neural network type function; therefore, GSW, in this case, is a pseudo-metric.

\vspace{0.5em}
\noindent
As mentioned in Section~\ref{sec:wasserstein_distance}, the result in Theorem~\ref{theorem:clt_random_projections} only applies to the $2$-Wasserstein distance. Thus, we only consider the GSW distance of the same order throughout
this paper.
\section{Polynomial Defining Function}
\label{sec:polynomial}
In this section, we consider the problem of finding a deterministic approximation for the generalized sliced-Wasserstein distance under the setting when the defining function $g$ is a polynomial function with an odd degree, which is defined as follows:
\begin{definition}[Polynomial defining function]
\label{def:polynomial_defining_function}
For a multi-index $\boldsymbol{\alpha}=(\alpha_1,\ldots,\alpha_d)\in\mathbb{N}^d$ and a vector $\boldsymbol{x}=(x_1,\ldots,x_d)\in\br^d$, we denote $|\boldsymbol{\alpha}|=\alpha_1+\ldots+\alpha_d$ and $\boldsymbol{x}^{\boldsymbol{\alpha}}=x_1^{\alpha_1}\ldots x_d^{\alpha_d}$. Then, a defining function of the form of a polynomial function with an odd degree $m$ is given by:
\begin{align*}
    g_{\mathsf{poly}}(\bx,\theta)=\sum_{|\boldsymbol{\alpha}|=m}\theta_{\boldsymbol{\alpha}}\boldsymbol{x}^{\boldsymbol{\alpha}},
\end{align*}
where $\theta:=(\theta_{\boldsymbol{\alpha}})_{|\boldsymbol{\alpha}|=m}\in\mathbb{S}^{q-1}$ with $q=\binom{m+d-1}{d-1}$ be the number of non-negative solutions to the equation $\alpha_1+\ldots+\alpha_d=m$. 
Accordingly, the Generalized Sliced-Wasserstein distance in this case is denoted as $\mathrm{poly-}GSW$.
\end{definition}
\vspace{0.5em}
\noindent
Subsequently, we introduce some necessary notations for our analysis. Let $\bX=(X_1,\ldots,X_d)^{\top}$ and $\bY=(Y_1,\ldots,Y_d)^{\top}$ be random vectors following probability distributions $\mu\in\mathcal{P}_2(\br^d)$ and $\nu\in\mathcal{P}_2(\br^d)$, respectively. For an odd positive integer $m\in\mathbb{N}$, by denoting $\mu_q$ and $\nu_q$ as the probability distributions in $\br^q$ of random vectors $\bU:=(\bX^{\balpha})_{|\balpha|=m}\in\br^q$ and $\bV:=(\bY^{\balpha})_{|\balpha|=m}\in\br^q$, we find that there is a connection between the GSW distance and the SW distance as follows:
\begin{proposition}
\label{prop:GSW_SW_relation}
Let $\mu,\nu\in\mathcal{P}_2(\br^d)$ be two probability measures in $\br^d$ with finite second moments and $\mu_q,\nu_q\in\mathcal{P}_2(\br^q)$ be defined as above where $q=\binom{m+d-1}{d-1}$ with $m\in\mathbb{N}$ is an odd positive integer. Then, we have:
\begin{align*}
    \mathsf{poly-}GSW_2(\mu,\nu) = SW_2(\mu_q,\nu_q).
\end{align*}
\end{proposition}
\begin{proof}[Proof of Proposition~\ref{prop:GSW_SW_relation}]
For $\theta\in\br^q$, we denote $g^{\theta}_{\mathsf{poly}}:\br^d\to\br$ as a function $x\mapsto g_{\mathsf{poly}}(\bx,\theta)$. It follows from the definition of $\mathsf{poly-}GSW$ distance that
\begin{align*}
    \mathsf{poly}-GSW_2^2(\mu,\nu)=&\int_{\br^q}W^2_2\Big((g^{\theta}_{\mathsf{poly}})_{\sharp}\mu,(g^{\theta}_{\mathsf{poly}})_{\sharp}\nu\Big)d\gamma_q(\theta)\\
    =&\int_{\br^q}W^2_2(\theta^{*}_{\sharp}\mu_q,\theta^{*}_{\sharp}\nu_q)d\gamma_q(\theta):=SW^2_2(\mu_q,\nu_q).
\end{align*}
Hence, we obtain the conclusion of this proposition.
\end{proof}
\vspace{0.5em}
\noindent
As a consequence, the original problem of approximating the $\mathsf{poly-}GSW$ distance between $\mu$ and $\nu$ boils down to estimating the SW distance between ${\mu}_q$ and ${\nu}_q$. Combining this result with Proposition~\ref{propo:approximate_sliced_by_wasserstein}, we obtain the following bound for $\mathsf{poly-}GSW(\mu,\nu)$.
\begin{theorem}
For any probability measures $\mu,\nu\in\mathcal{P}_2(\br^d)$ with finite second moments, there exists a universal constant $C>0$ such that
\begin{align*}
    \left|\mathsf{poly-}GSW_2(\mu,\nu)-q^{-\frac{1}{2}}\Big|\sqrt{\mathfrak{m}_2({\mu}_q)}-\sqrt{\mathfrak{m}_2({\nu}_q)}~\Big|\right| \leq C(\Xi_q(\mu_q) + \Xi_q(\nu_q))^{\frac{1}{2}},
\end{align*}
where $\mathfrak{m}_2(\zeta)$ and $\Xi_q(\zeta)$ are defined as in equation~\eqref{eq:xi_definition} for $\zeta\in\{\mu_q,\nu_q\}$. 
\end{theorem}
\vspace{0.5em}
\noindent
It is observed that components in the above approximation error, which are $\Xi_q(\mu_q)$ and $\Xi_q(\nu_q)$, cannot be shown to converge to 0 as $d\to\infty$ unless $\mu_q$ and $\nu_q$ are centered. We overcome this issue by using the following equality:
\begin{align*}
    \mathsf{poly-}GSW^2_2(\mu,\nu)= \mathsf{ poly-}GSW^2_2(\bar{\mu},\bar{\nu})+q^{-1}\|m_{\mu_q}-m_{\nu_q}\|^2,
\end{align*}
where $m_{\zeta},\bar{\zeta}$ are mean and centered versions of $\zeta$ for $\zeta\in\{\mu_q,\nu_q\}$. This equality is achieved by putting Proposition~\ref{propo:transform_to_the_centered} and Proposition~\ref{prop:GSW_SW_relation} together. Thus, we firstly try to approximate $\mathsf{poly-}GSW^2_2(\bar{\mu},\bar{\nu})$. It is sufficient to estimate $m_{\mu_q}$, $m_{\nu_q}$ and
\begin{align*}
    \mathfrak{m}_2(\bar{\mu}_q)&=\mathbb{E}[\|\bU\|^2]-\|\mathbb{E}[\bU]\|^2,\\
    \mathfrak{m}_2(\bar{\nu}_q)&=\mathbb{E}[\|\bV\|^2]-\|\mathbb{E}[\bV]\|^2.
\end{align*}
Let $\{\bx^{(j)}\}_{j=1}^N$ and $\{\by^{(j)}\}_{j=1}^N$ be samples drawn from probability distributions $\mu$ and $\nu$, respectively. Denote 
\begin{align*}
    \bU(\bx)=(\bU_i(\bx))_{i=1}^q&:=(\bx^{\balpha})_{|\balpha|=m},\\
    \bV(\by)=(\bV_i(\by))_{i=1}^q&:=(\by^{\balpha})_{|\balpha|=m},
\end{align*}
be two $q$-dimensional vectors. Then, the estimators of $\mathfrak{m}_2(\bar{\mu}_q)$, $\mathfrak{m}_2(\bar{\nu}_q)$, $m_{\mu_q}$ and $m_{\nu_q}$ can be calculated as:
\begin{align*}
    \widehat{\mathfrak{m}}_2(\bar{\mu}_q) &= \sum_{i=1}^q \Big(\frac{\sum_{j=1}^N\bU_i^2({\bx}^{(j)})}{N}-\Big(\frac{\sum_{j=1}^N\bU_i({\bx}^{(j)})}{N}\Big)^2\Big),\\
    \widehat{\mathfrak{m}}_2(\bar{\nu}_q) &= \sum_{i=1}^q \Big(\dfrac{\sum_{j=1}^N\bV_i^2({\by}^{(j)})}{N}-\Big(\dfrac{\sum_{j=1}^N\bV_i({\by}^{(j)})}{N}\Big)^2\Big),\\
    \widehat{m}_{\mu_q}&=\frac{1}{N}\sum_{j=1}^N\bU(\bx^{(j)}), \quad \widehat{m}_{\nu_q}=\frac{1}{N}\sum_{j=1}^N\bV(\by^{(j)}).
\end{align*}
\begin{corollary}
Consequently, an approximation of $\mathsf{poly-}GSW_2(\mu,\nu)$ can be written as
\begin{align*}
    \widehat{\mathsf{poly-}GSW}^2_2(\mu,\nu) = q^{-1}\left(\sqrt{{\widehat{\mathfrak{m}}}_2({\bar{\mu}}_q)}-\sqrt{{\widehat{\mathfrak{m}}}_2({\bar{\nu}}_q)}\right)^2   + q^{-1}\lVert \widehat{m}_{\mu_q} - \widehat{m}_{\nu_q}\rVert^2.
\end{align*}
\end{corollary}
\vspace{0.5em}
\noindent
To validate our approximation of $\mathsf{poly-}GSW(\mu,\nu)$, we provide in the following theorem an upper bound for the approximation error $(\Xi_q({\mu}_q)+\Xi_q({\nu}_q))^{\frac{1}{2}}$.
\begin{theorem}
\label{theorem:polynomial_error}
Let $(X_j)_{j \in \mathbb{N}}$ and $(Y_j)_{j \in \mathbb{N}}$ be sequences of independent random variables in $\br$ with zero means such that $\mathbb{E}[X_j^{4m}] < \infty$ and $\mathbb{E}[Y_j^{4m}] < \infty$ for all $j \in \mathbb{N}$, where $m$ is an odd positive integer. For $d \in \mathbb{N}$, let $\bX = \{X_j\}_{j = 1}^d$ and $\bY = \{Y_j\}_{j = 1}^d$ and denote by $\mu,\nu$ the distributions of $\bX,\bY$, respectively, while $\mu_q,\nu_q$ are defined as above. Then, we have
\begin{align}
\label{eqn:final_estimation_of_xi_polynomial}
    (\Xi_q(\mu_q)+\Xi_q(\nu_q))^{\frac{1}{2}}\leq \bigO(d^{-\frac{m}{8}}),
\end{align}
\end{theorem}
\begin{remark}
When $m=1$, the polynomial defining function reduces to the linear case $g_{\mathsf{poly}}(\bx,\theta)=\langle \bx,\theta\rangle$, leading to the fact that $\mathsf{poly-}GSW_p(\mu,\nu)=SW_p(\mu,\nu)$ for any $p\geq 1$ and $\mu,\nu\in\mathcal{P}_p(\br^d)$. Under that setting, the approximation error in Theorem~\ref{theorem:polynomial_error} approaches 0 at a rate of $d^{-\frac{1}{8}}$, which matches the result provided in [\cite{nadjahi2021fast}, Corollary 1].
\end{remark}
\vspace{0.5em}
\noindent
We now provide the proof of Theorem~\ref{theorem:polynomial_error}. 
\begin{proof}[Proof of Theorem~\ref{theorem:polynomial_error}]
Recall that
\begin{align}
    \Xi_q(\mu_q) = q^{-1}\Big\{A(\mu_q) + (\mathfrak{m}_2(\mu_q)B_1(\mu_q))^{1/2}+ \mathfrak{m}_2(\mu_q)^{1/5}B_2(\mu_q)^{4/5}\Big\}.\label{eq:xi_q_definition}
\end{align}
Thus, it is sufficient to bound $\mathfrak{m}_2(\mu_{q})$, $A(\mu_{q})$ and $B_k(\mu_{q})$ for $k \in \{1,2\}$. 
For any $\boldsymbol{\alpha}\in\br^d$ such that $|\boldsymbol{\alpha}| = m$, by using the H\"{o}lder's inequality, we have
\begin{align*}
    \mathbb{E}\left[(\boldsymbol{X}^{\balpha})^2\right] & = \mathbb{E}\left[X_1^{2\alpha_1}\ldots X_d^{2\alpha_d}\right] \\
    &\leq \mathbb{E}\left[X_1^{2m}\right]^{\alpha_1/m}\ldots\mathbb{E}\left[X_d^{2m}\right]^{\alpha_d/m} \\
    & \leq \max_{1\leq j\leq d} \mathbb{E}[X_j^{2m}],
\end{align*}
which leads to the following bound for $\mathfrak{m}_2(\mu_q)$:
\begin{align*}
    \mathfrak{m}_2(\mu_q) = \sum_{|\boldsymbol{\alpha}|=m}\mathbb{E}[({\boldsymbol{X}^{\boldsymbol{\alpha}}})^2]\leq q~\max_{1\leq j\leq d}\mathbb{E}[X_j^{2m}].
\end{align*}
Denote $\bU:=(\bX^{\balpha})_{|\balpha|=m}$ as a $q$-dimensional vector, then $[A(\mu_q)]^2=\mathbb{E}^2\Big|\|\bU\|^2-\mathbb{E}\|\bU\|^2\Big|$ is upper bounded by
\begin{align*}
    \Var[\lVert \bU\rVert^2]= \Var\left[\sum_{|\boldsymbol{\alpha}|=m}(\boldsymbol{X}^{\boldsymbol{\alpha}})^2\right]= \sum_{|\boldsymbol{\alpha}| = |\boldsymbol{\beta}| = m} \Cov\left[(\boldsymbol{X}^{\boldsymbol{\alpha}})^2, (\boldsymbol{X}^{\boldsymbol{\beta}})^2\right].
\end{align*}
Note that if two terms $\boldsymbol{X}^{\boldsymbol{\alpha}}$ and $\boldsymbol{X}^{\boldsymbol{\beta}}$ do not share any variables, then they are independent which implies that $\Cov\left[(\boldsymbol{X}^{\boldsymbol{\alpha}})^2, (\boldsymbol{X}^{\boldsymbol{\beta}})^2\right] = 0$. Otherwise, by denoting $\boldsymbol{\alpha} = (\alpha_1,\ldots,\alpha_d)$, $\boldsymbol{\beta} = (\beta_1,\ldots,\beta_d)$ and using the H\"{o}lder's inequality, $\left|\Cov\left[(\boldsymbol{X}^{\boldsymbol{\alpha}})^2, (\boldsymbol{X}^{\boldsymbol{\beta}})^2\right]\right|^2$ is bounded by
\begin{align*}
    \mathbb{E}\left[\left(\boldsymbol{X}^{\boldsymbol{\alpha}}\right)^4\right]\mathbb{E}\left[\left(\boldsymbol{X}^{\boldsymbol{\beta}}\right)^4\right] &= \mathbb{E}\left[X_1^{4\alpha_1}\ldots X_d^{4\alpha_d}\right] \mathbb{E}\left[X_1^{4\beta_1}\ldots X_d^{4\beta_d}\right]\\
    &\leq \left(\mathbb{E}\left[X_1^{4m}\right]\right)^{\alpha_1/m}\ldots\left(\mathbb{E}\left[X_d^{4m}\right]\right)^{\alpha_d/m}\times\left(\mathbb{E}\left[X_1^{4m}\right]\right)^{\beta_1/m}\ldots\left(\mathbb{E}\left[X_d^{4m}\right]\right)^{\beta_d/m}\\
    &\leq \max_{1\leq i \leq d} \left(\mathbb{E}\left[X_i^{4m}\right]\right)^2.
\end{align*}
Next, applying part (ii) of the Lemma \ref{lm:pair_of_poly} (cf. the end of this proof), we obtain that
\begin{align*}
    [A(\mu_q)]^2\leq\left( \dfrac{1}{[(m-1)!]^2}d^{2m-1} + \bigO(d^{2m-2})\right)\max_{1 \leq j \leq d} \mathbb{E}[X_j^{4m}].
\end{align*}
Finally, we need to bound $B_1(\mu_q)$ and $B_2(\mu_q)$. By utilizing the Cauchy-Schwarz inequality, we have $B_1(\mu_q) \leq B_2(\mu_q)$. Thus, it is sufficient to bound $B_2(\mu_q)$. Let us denote by $\boldsymbol{X}^{\prime}$ an independent copy of $\boldsymbol{X}$ and $\bU^{\prime}:=({\bX^{\prime}}^{\balpha})_{|\balpha|=m}$, then 
\vspace{-0.3cm}
\begin{align}
    \label{eqn:calculate_beta_in_polynomial}
    \langle \bU, \bU^{\prime}\rangle^2 = \left(\sum_{|\boldsymbol{\alpha}| = m}\boldsymbol{X}^{\boldsymbol{\alpha}}{\boldsymbol{X}^{\prime}}^{\boldsymbol{\alpha}}\right)^2= \sum_{|\boldsymbol{\alpha}| = m} (\boldsymbol{X}^{\boldsymbol{\alpha}}{\boldsymbol{X}^{\prime}}^{\boldsymbol{\alpha}})^2 + \sum_{\boldsymbol{\alpha} \neq \boldsymbol{\beta}}\boldsymbol{X}^{\boldsymbol{\alpha}}\boldsymbol{X}^{\boldsymbol{\beta}}{\boldsymbol{X}^{\prime}}^{\boldsymbol{\alpha}}{\boldsymbol{X}^{\prime}}^{\boldsymbol{\beta}}.
\end{align}
Noting that $\mathbb{E}[\boldsymbol{X}^{\boldsymbol{\alpha}}\boldsymbol{X}^{\boldsymbol{\beta}}] = 0$ if this product contains a variable of degree 1. Otherwise, using the H\"{o}lder inequality as in the above calculation, we have  $\mathbb{E}[\boldsymbol{X}^{\boldsymbol{\alpha}}\boldsymbol{X}^{\boldsymbol{\beta}}] \leq \max_{1\leq j\leq d} \mathbb{E}[X_j^{2m}]$ by the H\"{o}lder inequality. Thus, taking the expectation of equation~\eqref{eqn:calculate_beta_in_polynomial}, and using part (iii) of Lemma \ref{lm:pair_of_poly} (cf. the end of this proof), we have 
\begin{align*}
    \mathbb{E}\left[\langle \bU, \bU^{\prime}\rangle^2\right]
    \leq \bigO(d^{m})\left(\max_{1 \leq j\leq d}\mathbb{E}[X_j^{2m}]\right)^2.
\end{align*}
Therefore, $B_2(\mu_q) \leq \bigO(d^{m/2})\max_{1\leq j\leq d}\mathbb{E}[X_j^{2m}]$. In summary, we get
\begin{align}
\label{eqn:final_calculation_for_polynomial}
    &\mathfrak{m}_2(\mu_q) \leq \bigO(d^m)\max_{1 \leq j \leq d} \mathbb{E}[X_j^{2m}],\nonumber\\
    &A(\mu_q) \leq  \bigO(d^{m-1/2})\Big\{\max_{1 \leq j \leq d} \mathbb{E}[X_j^{4m}]\Big\}^{1/2},\nonumber\\
    &B_1(\mu_q)\leq B_2(\mu_q) \leq  \bigO(d^{m/2})\max_{1 \leq j\leq d}\mathbb{E}[X_j^{2m}].
\end{align}
Combining above results with equation~\eqref{eq:xi_q_definition} with a note that $q^{-1}=\bigO(d^{-m})$, we obtain that $\Xi_q(\mu_q)\leq\bigO(d^{-\frac{m}{4}})$. Similarly, we have $\Xi_q(\nu_q)\leq\bigO(d^{-\frac{m}{4}})$. Hence, we reach the conclusion of the theorem.
\end{proof}
\begin{lemma}
\label{lm:pair_of_poly}
For two positive integer numbers $d\in\mathbb{N}$ and $m\in\mathbb{N}$, let $P_d^m$ be the set of all multivariate polynomials of $x_1,\ldots,x_d$ of degree $m$. 
\begin{enumerate}
    \item[(i)] $|P_d^m| =\binom{d+m-1}{m} = \dfrac{d(d+1)\ldots (d+m-1)}{m!}$.
    \item[(ii)] Let $A_d^m$ be the set of all pairs of polynomials $P,Q\in P_d^m$ such that there is at least one variable appearing in the two polynomials. Then, there exists two polynomials $P_{A,m}^u$ and $P_{A,m}^\ell$ of variable $d$ with degree $2m-1$ and the leading coefficient $1/((m-1)!)^2$ such that $P_{A,m}^\ell(d) \leq |A_d^m| \leq P_{A,m}^u(d)$ for any $d\in\mathbb{N}$. 
    \item[(iii)] Let $B_d^m$ be the set of all pairs of polynomials $P,Q\in P_d^m$ such that the product of two polynomials in each pair does not contain a variable with degree 1. Then, there exist two polynomials $P_{B,m}^u$ and $P_{B,m}^\ell$ of variable $d$ with degree $m$ such that $P_{B,m}^\ell(d) \leq |B_d^m| \leq P_{B,m}^u(d)$. 
\end{enumerate}
\end{lemma}
\vspace{0.5em}
\noindent
The proof of Lemma~\ref{lm:pair_of_poly} is deferred to Appendix~\ref{appendix:pair_of_poly}.
\section{Neural Network Type Function}
\label{sec:neural_net}
Using the polynomial defining function in Section~\ref{sec:polynomial} yields an interesting approximation for the Generalized Sliced-Wasserstein distance, which generalizes the result in [\cite{nadjahi2021fast}, Theorem 1] for the linear case. However, the memory complexity of polynomial projections grows exponentially with the dimension of data $d$ and the degree of the polynomial, therefore, restricting their usage in machine learning and deep learning applications. To address that issue, we consider in this section the problem of approximating the Generalized Sliced-Wasserstein distance equipped with a neural network type defining function.

\vspace{0.5em}
\noindent
Prior to introducing the definition of a neural network type defining function, let us present some necessary notations for this section. Firstly, recall that $\bX=(X_1,\ldots,X_d)^{\top}$ and $\bY=(Y_1,\ldots,Y_d)^{\top}$ are random vectors having probability distributions $\mu\in\mathcal{P}_2(\br^d)$ and $\nu\in\mathcal{P}_2(\br^d)$, respectively. Then, let $n$ be the number of layers of a neural network, we denote by $\bTheta^{(1)},\ldots,\bTheta^{(n)}$ $n$ random matrices of size $d\times d$ such that they are independent of $\bX$ and $\bY$, and their entries are i.i.d random variables following a zero-mean Gaussian distribution $\mathcal{N}(0,d^{-1})$. Now, we are ready to define a neural network type defining function.
\begin{definition}[Neural network type defining function]
Let $\bx,\theta$ be two vectors in $\br^d$ and $\bTheta^{(1)},\ldots,\bTheta^{(n)}$ be random matrices defined as above, then a neural network type defining function is given by
\begin{align*}
    g_{\mathsf{neural}}(\bx,\theta)=\langle \theta, \bTheta^{(1)}\ldots\bTheta^{(n)}\bx\rangle.
\end{align*}
Accordingly, the Generalized Sliced-Wasserstein distance with this neural network type defining function is denoted as $\mathsf{neural-}GSW$.
\end{definition}
\vspace{0.5em}
\noindent
Next, we consider a random vector $\bX^*=(X^*_1,\ldots,X^*_d)^{\top}\in\br^d$ (resp. $\bY^*=(Y^*_1,\ldots,Y^*_d)^{\top}$) which is achieved by multiplying $n$ random matrices $\bTheta^{(1)},\ldots,\bTheta^{(n)}$ (each corresponding to a layer of the neural network) and $\bX$ (resp. $\bY$): 
\begin{align}
    \label{eq:x_neural_definition}
    \bX^*&=\bTheta^{(1)}\ldots\bTheta^{(n)}\bX,\\
    \label{eq:y_neural_definition}
    \bY^*&=\bTheta^{(1)}\ldots\bTheta^{(n)}\bY.
\end{align}
Based on the above definitions, we also attain in this section a relation between the neural-GSW distance and SW distance. In particular,
\begin{proposition}
\label{prop:neural_GSW_SW_relation}
Let $\mu^*$ and $\nu^*$ be two probability measures corresponding to the random vectors $\bX^*$ and $\bY^*$ given in equations~\eqref{eq:x_neural_definition} and \eqref{eq:y_neural_definition}, we obtain
\begin{align*}
    \mathsf{neural-}GSW_2(\mu,\nu)=SW_2(\mu^*,\nu^*).
\end{align*}
\end{proposition}
\begin{proof}[Proof of Proposition~\ref{prop:neural_GSW_SW_relation}]
For $\theta\in\br^d$, we denote $g^{\theta}_{\mathsf{neural}}:\br^d\to\br$ as a function $x\mapsto g_{\mathsf{neural}}(\bx,\theta)$. It follows from the definition of $\mathsf{poly-}GSW$ distance that
\begin{align*}
    \mathsf{neural}-GSW_2^2(\mu,\nu)=&\int_{\br^d}W^2_2\Big((g^{\theta}_{\mathsf{neural}})_{\sharp}\mu,(g^{\theta}_{\mathsf{neural}})_{\sharp}\nu\Big)d\gamma_d(\theta)\\
    =&\int_{\br^d}W^2_2(\theta^{*}_{\sharp}\mu^{*},\theta^{*}_{\sharp}\nu^*)d\gamma_d(\theta)\\
    =& SW^2_2(\mu^*,\nu^*).
\end{align*}
Thus, we reach the conclusion of this proposition.
\end{proof}
\vspace{0.5em}
\noindent
Consequently, deriving an approximation of the $\mathsf{neural-}GSW$ between $\mu$ and $\nu$ is equivalent to estimating the SW distance between ${\mu}^*$ and ${\nu}^*$. Putting this result and Proposition~\ref{propo:approximate_sliced_by_wasserstein} together, we achieve the following bound for $\mathsf{neural-}GSW(\mu,\nu)$. 
\begin{theorem}
For any probability measures $\mu,\nu\in\mathcal{P}_2(\br^d)$ with finite second moments, there exists a universal constant $C>0$ such that
\begin{align*}
    \left|\mathsf{neural-}GSW_2(\mu,\nu)-d^{-\frac{1}{2}}\big|\sqrt{\mathfrak{m}_2({\mu}^*)}-\sqrt{\mathfrak{m}_2({\nu}^*)}~\big|\right|\leq C(\Xi_d(\mu^*) + \Xi_d(\nu^*))^{\frac{1}{2}},
\end{align*}
where $\mathfrak{m}_2(\zeta)$ and $\Xi_d(\zeta)$ are defined as in equation~\eqref{eq:xi_definition} for $\zeta\in\{\mu^*,\nu^*\}$.
\end{theorem}
\vspace{0.5em}
\noindent
Subsequently, we estimate the values of $\mathfrak{m}_2(\mu^*)$ and $\mathfrak{m}_2(\nu^*)$.
Since $\bTheta^{(1)},\ldots,\bTheta^{(n)}$ are independent random matrices with zero means and they are independent of $\bX$ and $\bY$, it follows from equations~\eqref{eq:x_neural_definition} and \eqref{eq:y_neural_definition} that $\mathbb{E}[\bX^*]=0$ and $\mathbb{E}[\bY^*]=0$. In other words, $\mu^*$ and $\nu^*$ are zero-mean distributions. Therefore, let $\{\bx^{(j)}\}_{j=1}^N$ and $\{\by^{(j)}\}_{j=1}^N$ be samples drawn from probability distributions $\mu$ and $\nu$, respectively, we compute estimations of $\mathfrak{m}_2(\mu^*)$ and $\mathfrak{m}_2(\nu^*)$ as follows:
\begin{align*}
    \widehat{\mathfrak{m}}_2({{\mu^*}})  &= \frac{1}{N}\sum_{j=1}^N\|\bx^{(j)}\|^2, \quad \widehat{\mathfrak{m}}_2({{\nu^*}})  =\frac{1}{N}\sum_{j=1}^N\|\by^{(j)}\|^2.
\end{align*}

\begin{corollary}
As a consequence, an approximation of $\mathsf{neural-}GSW_2(\mu,\nu)$ can be written as
\begin{align*}
    \widehat{\mathsf{neural-}GSW}^2_2(\mu,\nu) = d^{-1}\left(\sqrt{{\widehat{\mathfrak{m}}}_2({{\mu^*}})}-\sqrt{{\widehat{\mathfrak{m}}}_2({{\nu}^*})}\right)^2.
\end{align*}
\end{corollary}
\vspace{0.5em}
\noindent
Finally, we provide in the following theorem an upper bound of the approximation error $(\Xi_d(\mu^*)+\Xi_d(\nu^*))^{\frac{1}{2}}$.
\begin{theorem}
\label{theorem:neural_error}
Let $(X_j)_{j \in \mathbb{N}}$ and $(Y_j)_{j \in \mathbb{N}}$ be sequences of independent random variables in $\br$ with zero means such that $\mathbb{E}[X_j^{4}] < \infty$ and $\mathbb{E}[Y_j^{4}] < \infty$ for all $j \in \mathbb{N}$. For $d \in \mathbb{N}$, let $\bX = \{X_j\}_{j = 1}^d$ and $\bY = \{Y_j\}_{j = 1}^d$ and denote by $\mu,\nu$ the distributions of $\bX,\bY$, respectively, while $\mu^*,\nu^*$ are defined as above. Then, we have
\begin{align}
\label{eqn:final_estimation_of_xi_neural}
    (\Xi_d(\mu^*)+\Xi_d(\nu^*))^{\frac{1}{2}}\leq \bigO(3^{\frac{n}{4}}d^{-\frac{1}{4}}+d^{-\frac{1}{8}}),
\end{align}
where $n\in\mathbb{N}$ is the number of neural network layers.
\end{theorem}
\begin{remark}
When there are no layers, i.e. $n=0$, the neural defining function reduces to the classic case $g_{\mathsf{neural}}(\bx,\theta)=\langle \bx,\theta\rangle$, implying that $\mathsf{neural-}GSW_p(\mu,\nu)=SW_p(\mu,\nu)$ for any $p\geq 1$ and $\mu,\nu\in\mathcal{P}_p(\br^d)$. Additionally, in this case, the approximation error in Theorem~\ref{theorem:neural_error} goes to 0 at the same rate as in [\cite{nadjahi2021fast}, Corollary 1], which is $\bigO(d^{-\frac{1}{8}})$.
\end{remark}
\begin{proof}[Proof sketch of Theorem~\ref{theorem:neural_error}]
The full proof of Theorem~\ref{theorem:neural_error} is in Appendix~\ref{appendix:neural_error}.
From the definition of $\Xi_d(\mu^*)$ in equation~\eqref{eq:xi_definition}, it is sufficient to bound $\mathfrak{m}_2(\mu^*)$, $A(\mu^*)$ and $B_k(\mu^*)$ for $k\in\{1,2\}$. Thus, this proof is divided into three parts as follows:

\vspace{0.5em}
\noindent
\textbf{Bounding $\mathfrak{m}_2(\mu^*)$:} As $\bX$ and $(\bTheta^{(i)})_{i=1}^n$ are independent, 
\begin{align}
    \mathfrak{m}_2(\mu^*)&=\sum_{j_0=1}^d\mathbb{E}\left(\sum_{j_1=1}^d\ldots\sum_{j_n=1}^d \Theta^{(1)}_{j_0j_1}\ldots \Theta^{(n)}_{j_{n-1}j_n} X_{j_n}\right)^2 \nonumber\\
    &=\sum_{j_0=1}^d \sum_{j_1=1}^d\ldots\sum_{j_n=1}^d\prod_{i=1}^n\mathbb{E}\left[(\Theta^{(i)}_{j_{i-1}j_{i}})^2\right]\mathbb{E}[X^2_{j_n}]\nonumber\\
    \label{eq:neural_bound_m}
    &= \sum_{j_n=1}^d\mathbb{E}[X^2_{j_n}] \leq d\max_{1 \leq j \leq d} \mathbb{E}[X_{j}^2]= \bigO(d).
\end{align}

\vspace{0.5em}
\noindent
\textbf{Bounding $A(\mu^*)$:} It follows from the definition of $A(\mu^*)$ in equation~\eqref{eq:xi_definition} that $[A(\mu^*)]^2\leq\Var(\|\bX^*\|^2)$. Thus, $A(\mu^*)$ is upper bounded by
\begin{align}
    &\Var\left[\sum_{j_0=1}^d\left(\sum_{j_1=1}^d\ldots\sum_{j_n=1}^d \Theta^{(1)}_{j_0j_1}\ldots \Theta^{(n)}_{j_{n-1}j_n} X_{j_n}\right)^2\right] \nonumber\\
    =&~d\Var\left[\left(\sum_{j_1=1}^d\ldots\sum_{j_n=1}^d \Theta^{(1)}_{1j_1}\ldots \Theta^{(n)}_{j_{n-1}j_n} X_{j_n}\right)^2\right] \nonumber\\
    \label{eq:var_general}
    +& d(d-1)\Cov\Biggl[\left(\sum_{j_1=1}^d\ldots\sum_{j_n=1}^d \Theta^{(1)}_{1j_1}\ldots \Theta^{(n)}_{j_{n-1}j_n} X_{j_n}\right)^2,\left(\sum_{j_1=1}^d\ldots\sum_{j_n=1}^d \Theta^{(1)}_{2j_1}\ldots \Theta^{(n)}_{j_{n-1}j_n} X_{j_n}\right)^2\Biggl].
\end{align}
\textbf{Regarding the variance term in equation~\eqref{eq:var_general}:} Since $\bX$ and $(\bTheta^{(i)})_{i=1}^n$ are independent, this term is equal to
\begin{align}
\label{eq:cov_with_same_first_coefficient}
&\sum \Cov\Big[\Theta^{(1)}_{j_0^1j_1^1}\ldots\Theta^{(n)}_{j_{n-1}^1j_n^1}X_{j_n^1}\Theta^{(1)}_{j_0^2j_1^2}\ldots\Theta^{(n)}_{j_{n-1}^2j_n^2}X_{j_n^2},\nonumber\\
&\quad\Theta^{(1)}_{j_0^3j_1^3}\ldots\Theta^{(n)}_{j_{n-1}^3j_n^3}X_{j_n^3}\Theta^{(1)}_{j_0^4j_1^4}\ldots\Theta^{(n)}_{j_{n-1}^4j_n^4}X_{j_n^4}\Big],
\end{align}
where the sum is subject to tuples $(\boldsymbol{j}^1,\boldsymbol{j}^2,\boldsymbol{j}^3,\boldsymbol{j}^4)$ with  $\boldsymbol{j}^i = (j_0^i, \ldots, j_n^i)$ such that $j_0^i = 1$.  It is easy to see that if $(\boldsymbol{j}^1,\boldsymbol{j}^2,\boldsymbol{j}^3,\boldsymbol{j}^4)$ is not a bad tuple (defined in Lemma~\ref{lem:bad_tuple}), the respected covariance value in equation~\eqref{eq:cov_with_same_first_coefficient} equals 0. Thus, by utilizing part (i) of Lemma~\ref{lem:bad_tuple} (cf. the end of this proof), the variance term in equation~\eqref{eq:var_general} is bounded by $\bigO(3^n)$.

\vspace{0.5em}
\noindent
\textbf{Regarding the covariance term in equation~\eqref{eq:var_general}:} This term is equal to
\begin{align}
\label{eq:cov_with_different_coefficient}
&\sum \Cov\Big[\Theta^{(1)}_{j_0^1j_1^1}\ldots\Theta^{(n)}_{j_{n-1}^1j_n^1}X_{j_n^1}\Theta^{(1)}_{j_0^2j_1^2}\ldots\Theta^{(n)}_{j_{n-1}^2j_n^2}X_{j_n^2},\nonumber\\
&\qquad\Theta^{(1)}_{j_0^3j_1^3}\ldots\Theta^{(n)}_{j_{n-1}^3j_n^3}X_{j_n^3}\Theta^{(1)}_{j_0^4j_1^4}\ldots\Theta^{(n)}_{j_{n-1}^4j_n^4}X_{j_n^4}\Big],
\end{align}
where the sum is subject to tuples $(\boldsymbol{j}^1,\boldsymbol{j}^2,\boldsymbol{j}^3,\boldsymbol{j}^4)$ with $\boldsymbol{j}^i = (j_0^i, \ldots, j_n^i)$ such that $j_0^1 = j_0^2 = 1$, $j_0^3 = j_0^4 = 2$.  It is obvious that if $(\boldsymbol{j}^1,\boldsymbol{j}^2,\boldsymbol{j}^3,\boldsymbol{j}^4)$ is not a bad tuple as in the definition of Lemma~\ref{lem:bad_tuple} (cf. the end of this proof), the respected value of the covariance term in equation~\eqref{eq:cov_with_different_coefficient} equals 0. Therefore, by using part (ii) of Lemma~\ref{lem:bad_tuple}, the covariance term in equation~\eqref{eq:var_general} is shown to be bounded by $\bigO(3^n)/d$. 

\vspace{0.5em}
\noindent
Putting the above results together, we obtain that
\begin{align}
    \label{eq:neural_bound_A}
    A(\mu^*)\leq\bigO(3^{\frac{n}{2}}d^{\frac{1}{2}}).
\end{align}
\textbf{Bounding $B_k(\mu^*)$:} By the Cauchy-Schwartz inequality, we have $B_1(\mu^*)\leq B_2(\mu^*)$. Therefore, it is sufficient to bound $B_2(\mu^*)$. Denote by ${\bX^*}^{\prime}$ an independent copy of $\bX^*$, we consider
\begin{align*}
    \langle \bX^*,{\bX^*}^{\prime}\rangle^2=\sum_{j=1}^d(X^*_j)^2({X^*_j}^{\prime})^2+\sum_{i\neq j}X^*_iX^*_j{X^*_i}^{\prime}{X^*_j}^{\prime}.
\end{align*}
As $\bX^*$ and ${\bX^*}^{\prime}$ are independent, we have
\begin{align*}
    \EE\Big[(X^*_j)^2({X^*_j}^{\prime})^2\Big] =(\EE[(X^*_j)^2])^2=&\left(\mathbb{E}\left[\left(\sum_{j_1=1}^d\ldots\sum_{j_n=1}^d \Theta^{(1)}_{jj_1}\ldots \Theta^{(n)}_{j_{n-1}j_n} X_{j_n}\right)^2\right]\right)^2\\
    =& \left(\mathbb{E}\left[\sum_{j_1=1}^d\ldots\sum_{j_n=1}^d \left(\Theta^{(1)}_{jj_1}\right)^2\ldots \left(\Theta^{(n)}_{j_{n-1}j_n}\right)^2 X^2_{j_n}\right]\right)^2 \\
    \leq& \left(d^n \max_{1\leq j \leq d}\dfrac{1}{d^n} \mathbb{E}[X_j^2]\right)^2 = (\max_{1\leq j\leq d}\EE[(X^*_j)^2])^2,
\end{align*}
and 
$\EE[X^*_iX^*_j{X^*_i}^{\prime}{X^*_j}^{\prime}]=(\EE[X^*_iX^*_j])^2=0$ for $i\neq j$, 
Thus, we get
\begin{align}
    \label{eq:neural_bound_B}
    B_1(\mu^*)\leq B_2(\mu^*)\leq \bigO(d^{\frac{1}{2}}).
\end{align}
From equations~\eqref{eq:neural_bound_m}, \eqref{eq:neural_bound_A} and \eqref{eq:neural_bound_B}, we obtain
\begin{align*}
    \Xi_d(\mu^*)\leq\bigO(3^{\frac{n}{2}}d^{-\frac{1}{2}}+d^{-\frac{1}{4}}).
\end{align*}
Similarly, we have $\Xi_d(\nu^*)\leq\bigO(3^{\frac{n}{2}}d^{-\frac{1}{2}}+d^{-\frac{1}{4}})$.
Hence, we reach the conclusion of the theorem.
\end{proof}

\begin{lemma}
\label{lem:bad_tuple}
Let us consider tuples $((j_0^1,j_1^1,\ldots,j_n^1)$, $(j_0^2,j_1^2,\ldots,j_n^2)$, $(j_0^3,j_1^3,\ldots,j_n^3)$, $(j_0^4,j_1^4,\ldots,j_n^4))$, where $j_i^k$ is a positive integer number in $[1,d]$. 
We call a tuple to be \textbf{bad} if for each $i\in\{0,1\ldots,n\}$, there exists a way to divide the set $\{1,2,3,4\}$ into two disjoint subsets $\{p,q\}$ and $\{r,s\}$ such that $j_i^{p} = j_i^{q}$, $j_{i+1}^{p} = j_{i+1}^{q}$, $j_i^{r} = j_i^{s}$, $j_{i+1}^{r} = j_{i+1}^{s}$. Then,
\begin{itemize}
    \item [(i)] The number of bad tuples such that $j_0^1 = j_0^2 = j_0^3 = j_0^4 = 1$ is a polynomial of variable $d$ of degree $2n$ with the highest coefficient of 1. 
    \item [(ii)] The number of bad tuples such that $j_0^1 = j_0^2 = 1, j_0^3 = j_0^4 = 2$ is a polynomial of variable $d$ of degree $2n$ with the highest coefficient of 1. 
\end{itemize}
\end{lemma}
\vspace{0.5em}
\noindent
The proof of Lemma~\ref{lem:bad_tuple} is deferred to Appendix~\ref{appendix:bad_tuple}.

\section{Experiments}
\label{sec:experiments}
In this section, we focus on testing the approximation error of our proposed approximated GSW. In particular, we try to increase the dimension of simulated data to see the change of the approximation error, namely, the $\mathbb{L}_1$ distance between the approximated GSW and the Monte Carlo GSW with a huge number of projections, e.g., 20000. For all experiments, we repeat the process 100 times and report the mean and the standard deviation each time.

\begin{figure*}[t]
\begin{center}
    
  \begin{tabular}{ccc}
  \widgraph{0.32\textwidth}{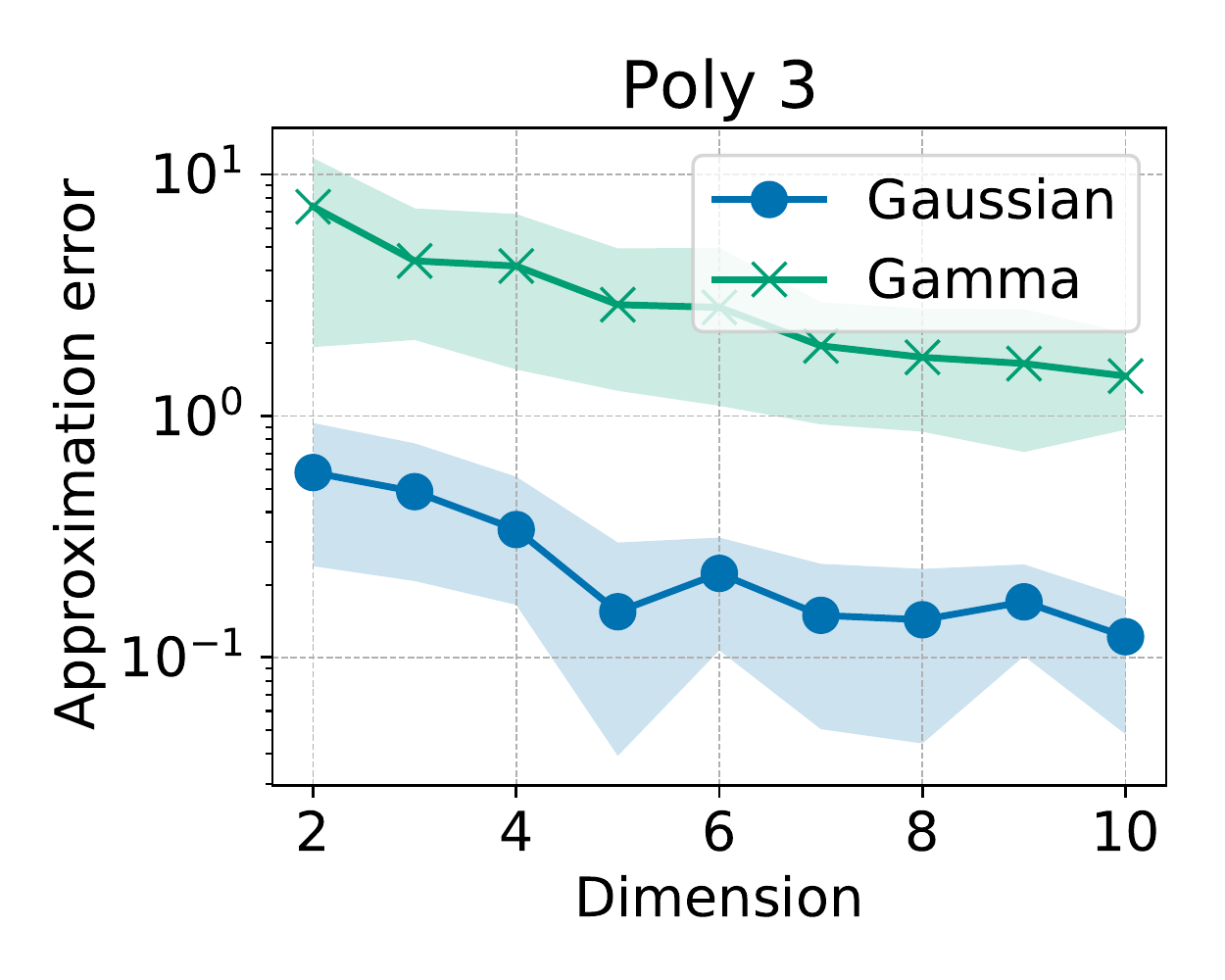} 
  &
  
  \widgraph{0.32\textwidth}{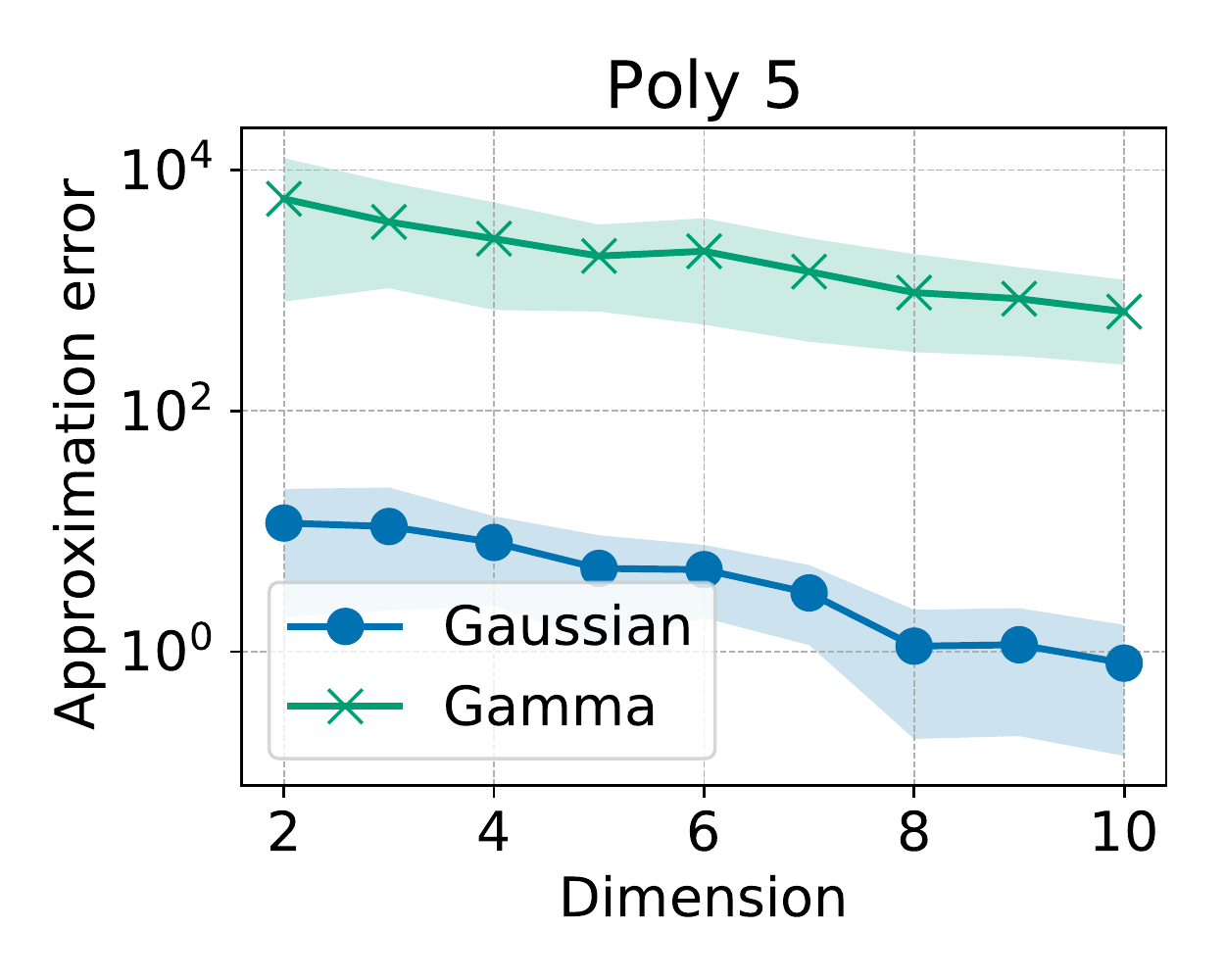} 
   &
  \widgraph{0.32\textwidth}{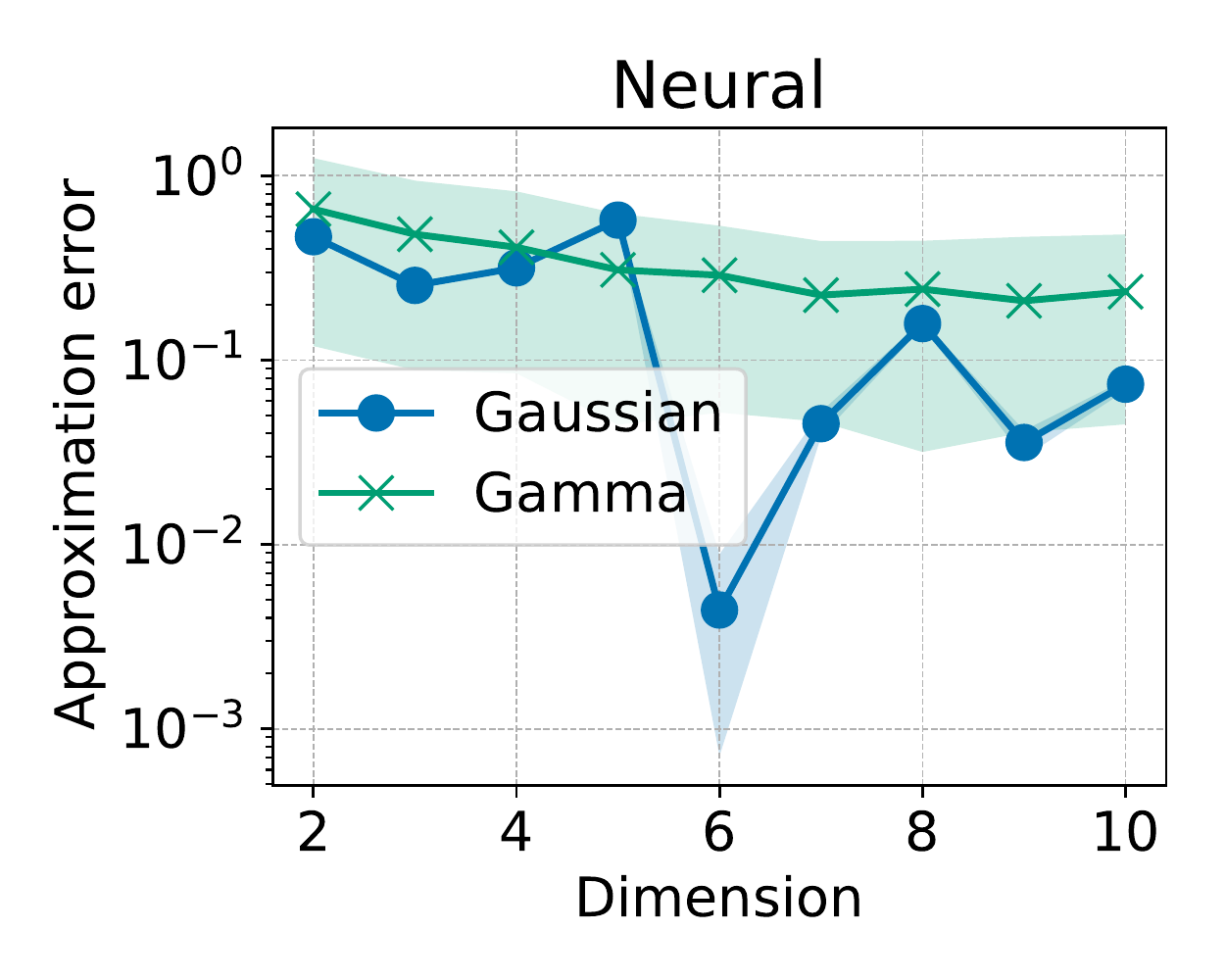} 
  \end{tabular}
  \end{center}
  \vskip -0.3in
  \caption{
  \footnotesize{Approximation error  between approximated GSW with the Monte Carlo GSW with a huge number of projections between empirical distributions on samples that are drawn from Multivariate Gaussian distributions and Gamma distributions. 
}
} 
  \label{fig:approximation_error}
  \vskip -0.1in
\end{figure*}

\vspace{0.5em}
\noindent
\textbf{Approximation error on Multivariate Gaussian and Gamma:}  In this setup, we first  generate two sets of $n=10^4$ $d$-dimensional samples from two Multivariate Gaussian distributions $\mathcal{N}(\mathbf{0},\II_d)$ and $\mathcal{N}(\mathbf{1},2\II_{d})$. We denote two empirical distributions as $\mu_d=\frac{1}{n} \sum_{i=1}^n \delta_{x_i}$ and $\nu_d = \frac{1}{n}\sum_{i=1}^n \delta_{y_i}$. We then compute our approximated GSW and the Monte Carlo GSW with the polynomial defining function (degree 3 and 5) and the neural defining function. Finally, we plot the approximation error with respect to the number of dimensions in Figure~\ref{fig:approximation_error}. From the figure, we observe that the approximation error has a decreasing trend when the number of dimensions increases for all defining functions. A similar phenomenon happens when we use empirical samples from multivariate random variables where each dimension follows $\mathrm{Gamma}(1,2)$ and $\mathrm{Gamma}(1,3)$. We also observe that the error in the Gamma case is larger than in the Gaussian case. The reason is that our approximation is based on the  closed form of Wasserstein distance between two Gaussian distributions.

\vspace{0.5em}
\noindent
\textbf{Approximation error on autoregressive processes of order one (AR(1)): } We would like to recall that in AR(1) process, $X_t = \alpha X_{t-1} + \varepsilon_t$ where $\alpha \in [0,1]$ and $\{\varepsilon_i\}_{i=1}^n $ are  i.i.d. real random variables with $\mathbb{E}[\varepsilon_i]=0$ that have finite second-order moment. We use this process with $10^4+d$ to generate samples. We only take the last $d$ steps while the previous steps are for ``burn in'' that guarantees the stationary solution of the process. We generate empirical samples $\{x_i\}_{i=1}^n$ and $\{y_i\}_{i=1}^n$ using the
same Gaussian noise (Student noise). We report the approximation errors for different values of $\alpha$ and defining functions in Figure~\ref{fig:armodel}. Similar to the previous experiment, the approximation error decreases when the number of dimensions increases. 
\section{Conclusion}
\label{sec:conclusion}
In this paper, we establish deterministic and fast approximations of the generalized sliced Wasserstein distance without using random projections by leveraging the conditional central limit theorem for Gaussian projections. In both cases of polynomial defining function and neural network type function, we provide a rigorous guarantee that under some mild assumptions on two input probability measures, the approximation errors approach zero when the dimensions increases. The analysis of error for the circular function case is left for future work. Finally, we carry out some simulation studies on different types of probability distributions to justify our theoretical result.
\section*{Acknowledgements}
Nhat Ho
acknowledges support from the NSF IFML 2019844 and the NSF AI Institute for Foundations of Machine Learning.
\begin{figure*}[!t]
\begin{center}
    
  \begin{tabular}{ccc}
  \widgraph{0.32\textwidth}{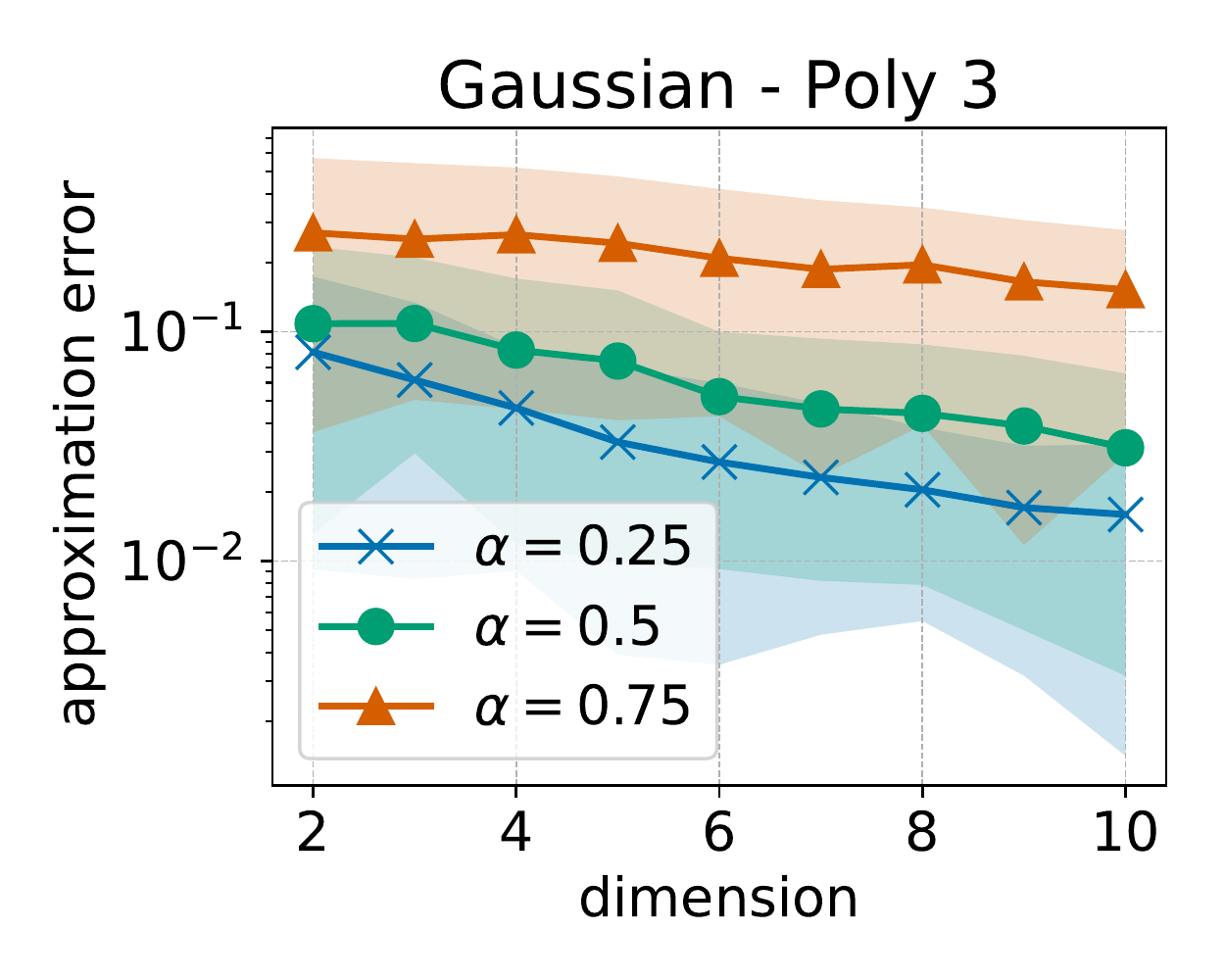} 
  &
  \widgraph{0.32\textwidth}{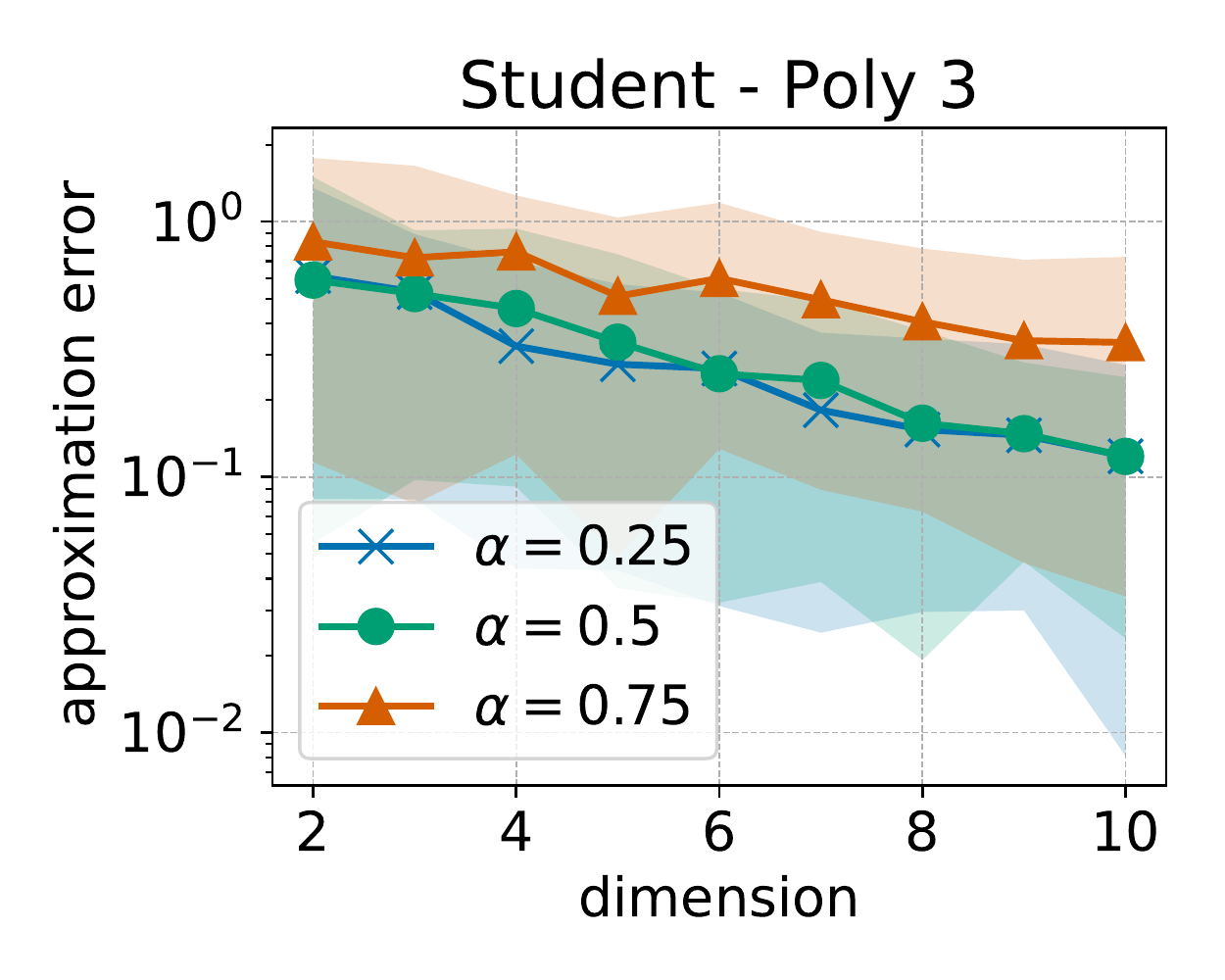} 
  &
  \widgraph{0.32\textwidth}{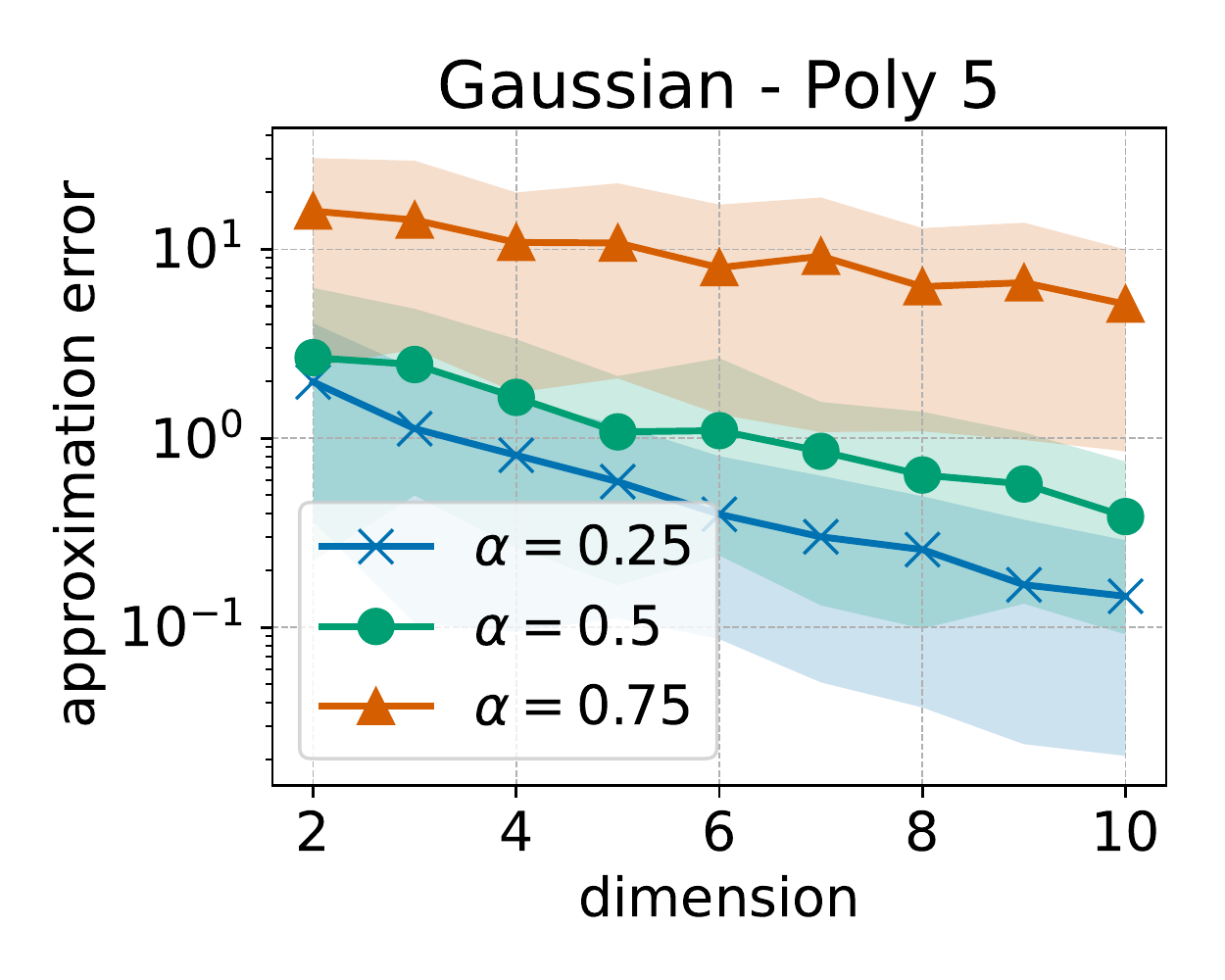} 
  \\
  \widgraph{0.32\textwidth}{images/poly_3_student_ar.pdf} 
   &
  \widgraph{0.32\textwidth}{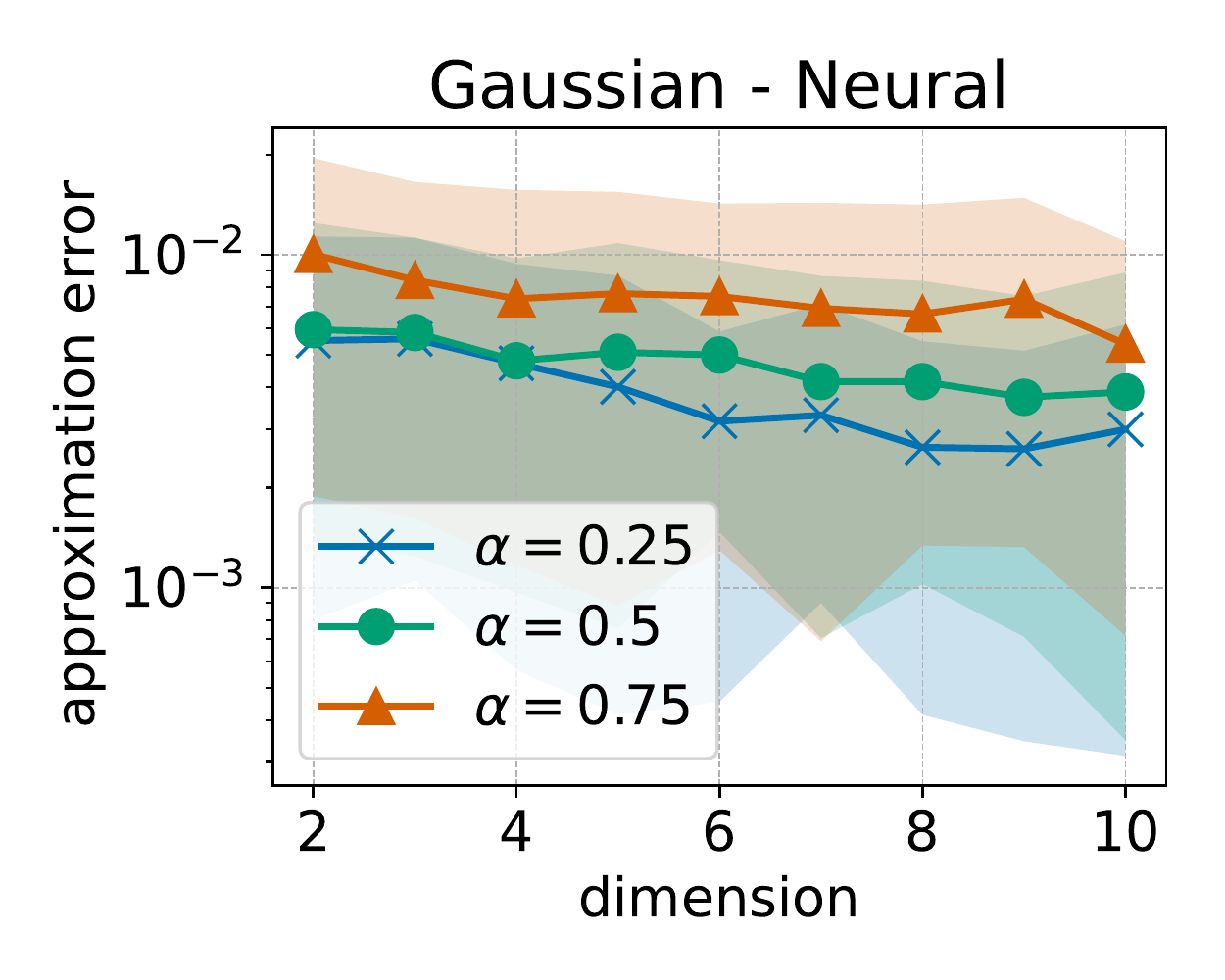} 
  &
  \widgraph{0.32\textwidth}{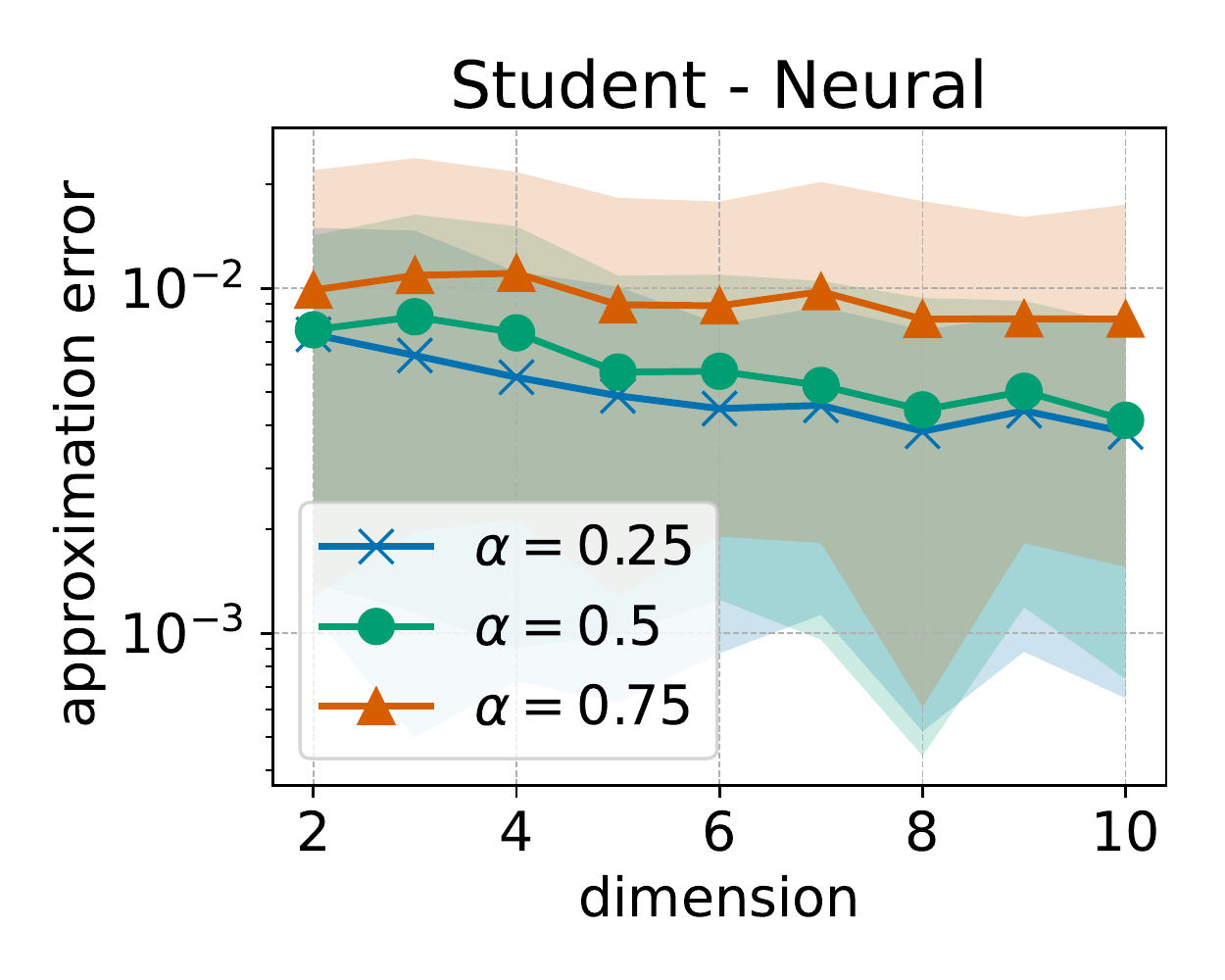} 
  \end{tabular}
  \end{center}
  \vskip -0.3in
  \caption{
  \footnotesize{Approximation error  between approximated GSW with the Monte Carlo GSW with a huge number of projections between empirical distributions on samples that are drawn from utoregressive processes of order one AR(1).
}
} 
  \label{fig:armodel}
  \vskip -0.2in
\end{figure*}

\appendix

\begin{center}
\textbf{\Large{Supplementary Materials for
``Fast Approximation of the Generalized Sliced-Wasserstein Distance''}}
\end{center}
In this supplementary material, we firstly provide proofs of remaining results in Appendix~\ref{appendix:missing_proofs}. Then, we present the approximation of the generalized sliced Wasserstein distance when the defining function is circular in Appendix~\ref{sec:circular}.
\section{Missing Proofs}
\label{appendix:missing_proofs}
This appendix is devoted to show the proofs of Theorem~\ref{theorem:neural_error}, Lemma~\ref{lm:pair_of_poly} and Lemma~\ref{lem:bad_tuple}.
\subsection{Proof of Theorem~\ref{theorem:neural_error}}
\label{appendix:neural_error}
Given the definition of $\Xi_d(\mu^*)$ as follows:
\begin{align}
    \label{eqn:xi_star_def}
    \Xi_d(\mu^*) = d^{-1}\Big\{A(\mu^*) + [\mathfrak{m}_2(\mu^*)B_1(\mu^*)]^{1/2}+ \mathfrak{m}_2(\mu^*)^{1/5}B_2(\mu^*)^{4/5}\Big\}, 
\end{align}
this proof consists of bounding $\mathfrak{m}_2(\mu^*)$, $A(\mu^*)$, and $B_k(\mu^*)$ for $k \in \{1,2\}$.

\vspace{0.5em}
\noindent
\textbf{An upper bound of $\mathfrak{m}_2(\mu^*)$:} Firstly, let us recall the definition of $\mathfrak{m}_2(\mu^*)$:
\begin{align*}
    \mathfrak{m}_2(\mu^*)&=\mathbb{E}[\|\bX^*\|^2]=\mathbb{E}[\|\bTheta^{(1)}\ldots\bTheta^{(n)}\bX\|^2]=\sum_{j_0=1}^d\mathbb{E}\left[\left(\sum_{j_1=1}^d\ldots\sum_{j_n=1}^d \Theta^{(1)}_{j_0j_1}\ldots \Theta^{(n)}_{j_{n-1}j_n} X_{j_n}\right)^2\right].
\end{align*}
Since $\bX$ and $\bTheta^{(1)},\ldots,\bTheta^{(n)}$ are independent, we have
\begin{align*}
    \mathfrak{m}_2(\mu^*)&= \sum_{j_0=1}^d\mathbb{E}\left[\sum_{j_1=1}^d\ldots\sum_{j_n=1}^d (\Theta^{(1)}_{j_0j_1})^2\ldots (\Theta^{(n)}_{j_{n-1}j_n})^2X_{j_n}^2\right] \\
    &=\sum_{j_0=1}^d \sum_{j_1=1}^d\ldots\sum_{j_n=1}^d\prod_{i=1}^n\mathbb{E}\left[(\Theta^{(i)}_{j_{i-1}j_{i}})^2\right]\mathbb{E}[X^2_{j_n}]\\
    &= d^{n}\times\dfrac{1}{d^n}\sum_{j_n=1}^d\mathbb{E}[X^2_{j_n}] \\
    &= \sum_{j_n=1}^d\mathbb{E}[X^2_{j_n}]\\
    &\leq d\max_{1 \leq j \leq d} \mathbb{E}[X_{j}^2]= \bigO(d).
\end{align*}
\textbf{An upper bound of $A(\mu^*)$:}
It is worth noting that
\begin{align*}
    [A(\mu^*)]^2 =\left[\mathbb{E}\Big|\|\bX^*\|^2-\mathbb{E}\|\bX^*\|^2\Big|\right]^2 & \leq \mathbb{E}\left[\|\bX^*\|^2-\mathbb{E}\|\bX^*\|^2\right]^2\\
    &=\Var(\|\bX^*\|^2)\\
    &=\Var\left[\sum_{i=1}^d (X_i^{*})^2\right].
\end{align*}
As a result, $[A(\mu^*)]^2$ is upper bounded by
\begin{align}
\label{eqn:estimate_alpha_neural_network}
    &\Var\left[\sum_{j_0=1}^d\left(\sum_{j_1=1}^d\ldots\sum_{j_n=1}^d \Theta^{(1)}_{j_0j_1}\ldots \Theta^{(n)}_{j_{n-1}j_n} X_{j_n}\right)^2\right] \nonumber\\
    &= \sum_{j_0=1}^d\Var\left[\left(\sum_{j_1=1}^d\ldots\sum_{j_n=1}^d \Theta^{(1)}_{j_0j_1}\ldots \Theta^{(n)}_{j_{n-1}j_n} X_{j_n}\right)^2\right] \nonumber\\
    &+ 2\sum_{1\leq j_0 < j_0' \leq d} \Cov\left[\left(\sum_{j_1=1}^d\ldots\sum_{j_n=1}^d \Theta^{(1)}_{j_0j_1}\ldots \Theta^{(n)}_{j_{n-1}j_n} X_{j_n}\right)^2, \left(\sum_{j_1=1}^d\ldots\sum_{j_n=1}^d \Theta^{(1)}_{j_0'j_1}\ldots \Theta^{(n)}_{j_{n-1}j_n} X_{j_n}\right)^2\right] \nonumber\\
    &= d\Var\left[\left(\sum_{j_1=1}^d\ldots\sum_{j_n=1}^d \Theta^{(1)}_{1j_1}\ldots \Theta^{(n)}_{j_{n-1}j_n} X_{j_n}\right)^2\right] \nonumber\\
    &+ d(d-1)\Cov\left[\left(\sum_{j_1=1}^d\ldots\sum_{j_n=1}^d \Theta^{(1)}_{1j_1}\ldots \Theta^{(n)}_{j_{n-1}j_n} X_{j_n}\right)^2, \left(\sum_{j_1=1}^d\ldots\sum_{j_n=1}^d \Theta^{(1)}_{2j_1}\ldots \Theta^{(n)}_{j_{n-1}j_n} X_{j_n}\right)^2\right],
\end{align}
where the last equality is due to the fact that $\bTheta^{(1)},\ldots,\bTheta^{(n)}$ are i.i.d random matrices.

\vspace{0.5em}
\noindent
We now try to bound the variance term in equation~\eqref{eqn:estimate_alpha_neural_network}. As $\bTheta^{(1)},\ldots,\bTheta^{(n)}$ are independent of $\bX^*$, we get that $\Var\left[\left(\sum_{j_1=1}^d\ldots\sum_{j_n=1}^d \Theta^{(1)}_{1j_1}\ldots \Theta^{(n)}_{j_{n-1}j_n} X_{j_n}\right)^2\right]$ is equal to
\begin{align}
\label{eqn:cov_with_same_first_coefficient}
\sum_{(\boldsymbol{j}^1,\boldsymbol{j}^2,\boldsymbol{j}^3,\boldsymbol{j}^4)} \Cov\Big[\Theta^{(1)}_{j_0^1j_1^1}\ldots\Theta^{(n)}_{j_{n-1}^1j_n^1}X_{j_n^1}\Theta^{(1)}_{j_0^2j_1^2}\ldots\Theta^{(n)}_{j_{n-1}^2j_n^2}X_{j_n^2},\Theta^{(1)}_{j_0^3j_1^3}\ldots\Theta^{(n)}_{j_{n-1}^3j_n^3}X_{j_n^3}\Theta^{(1)}_{j_0^4j_1^4}\ldots\Theta^{(n)}_{j_{n-1}^4j_n^4}X_{j_n^4}\Big],
\end{align}
where $(\boldsymbol{j}^1,\boldsymbol{j}^2,\boldsymbol{j}^3,\boldsymbol{j}^4)$ is a tuple of $\boldsymbol{j}^i = (j_0^i, \ldots, j_n^i)$ such that $j_0^i = 1$.  It can be seen that if $(\boldsymbol{j}^1,\boldsymbol{j}^2,\boldsymbol{j}^3,\boldsymbol{j}^4)$ is not a bad tuple as in the definition of Lemma~\ref{lem:bad_tuple}, the respected covariance value in the sum in equation~\eqref{eqn:cov_with_same_first_coefficient} is equal to 0. However, when $(\boldsymbol{j}^1,\boldsymbol{j}^2,\boldsymbol{j}^3,\boldsymbol{j}^4)$ is a bad tuple, we have
\begin{align}
\label{eqn:covariance_in_bad_case}
&\Cov\left[\Theta^{(1)}_{j_0^1j_1^1}\ldots\Theta^{(n)}_{j_{n-1}^1j_n^1}X_{j_n^1}\Theta^{(1)}_{j_0^2j_1^2}\ldots\Theta^{(n)}_{j_{n-1}^2j_n^2}X_{j_n^2},
\Theta^{(1)}_{j_0^3j_1^3}\ldots\Theta^{(n)}_{j_{n-1}^3j_n^3}X_{j_n^3}\Theta^{(1)}_{j_0^4j_1^4}\ldots\Theta^{(n)}_{j_{n-1}^4j_n^4}X_{j_n^4}\right]\nonumber\\
\leq& \left(\mathbb{E}\left[\left(\Theta^{(1)}_{j_0^1j_1^1}\ldots\Theta^{(n)}_{j_{n-1}^1j_n^1}X_{j_n^1}\Theta^{(1)}_{j_0^2j_1^2}\ldots\Theta^{(n)}_{j_{n-1}^2j_n^2}X_{j_n^2}\right)^2\right]\right)^{1/2}\nonumber\\
&\hspace{4em}\times\left(\mathbb{E}\left[\left(\Theta^{(1)}_{j_0^3j_1^3}\ldots\Theta^{(n)}_{j_{n-1}^3j_n^3}X_{j_n^3}\Theta^{(1)}_{j_0^4j_1^4}\ldots\Theta^{(n)}_{j_{n-1}^4j_n^4}X_{j_n^4}\right)^2\right]\right)^{1/2}\nonumber \\
=& \left(\prod_{i=1}^n \left(\mathbb{E}\left[\left(\Theta^{(i)}_{j_{i-1}^1j_i^1}\Theta^{(i)}_{j_{i-1}^2j_i^2}\right)^2\right]\mathbb{E}\left[\left(\Theta^{(i)}_{j_{i-1}^3j_i^3}\Theta^{(i)}_{j_{i-1}^4j_i^4}\right)^2\right]\right)^{1/2}\right)\left(\mathbb{E}\left[X^2_{j_n^1}X^2_{j_n^2}\right]\mathbb{E}\left[X^2_{j_n^3}X^2_{j_n^4}\right]\right)^{1/2}\nonumber\\
\leq & \left(\prod_{i=1}^n \left(\mathbb{E}\left[(\Theta^{(i)}_{j_{i-1}^1j_i^1})^4\right]\mathbb{E}\left[(\Theta^{(i)}_{j_{i-1}^2j_i^2})^4\right]\mathbb{E}\left[(\Theta^{(i)}_{j_{i-1}^3j_i^3})^4\right]\mathbb{E}\left[(\Theta^{(i)}_{j_{i-1}^4j_i^4})^4\right]\right)^{1/4}\right)\nonumber\\
&\hspace{4em}\times\left(\mathbb{E}\left[X^4_{j_n^1}\right]\mathbb{E}\left[X^4_{j_n^2}\right]\mathbb{E}\left[X^4_{j_n^3}\right]\mathbb{E}\left[X^4_{j_n^4}\right]\right)^{1/4}\nonumber\\
\leq & ~\dfrac{3^n}{d^{2n}}\max_{1\leq j\leq d} \mathbb{E}\left[X_j^4\right].
\end{align}
Using part $(i)$ of Lemma \ref{lem:bad_tuple} about the number of bad tuples, we obtain 
\begin{equation}
\label{eqn:bound_when_the_same_first_coefficient}
    \Var\left[\left(\sum_{j_1=1}^d\ldots\sum_{j_n=1}^d \Theta^{(1)}_{1j_1}\ldots \Theta^{(n)}_{j_{n-1}j_n} X_{j_n}\right)^2\right]\leq\bigO(3^n)\max_{1\leq j\leq d} \mathbb{E}\left[X_j^4\right].
\end{equation}
\vspace{0.5em}
\noindent
Subsequently, we will bound the covariance term in equation~\eqref{eqn:estimate_alpha_neural_network}.
\begin{align}
\label{eqn:cov_with_different_coefficient}
&\Cov\left[\left(\sum_{j_1=1}^d\ldots\sum_{j_n=1}^d \Theta^{(1)}_{1j_1}\ldots \Theta^{(n)}_{j_{n-1}j_n} X_{j_n}\right)^2, \left(\sum_{j_1=1}^d\ldots\sum_{j_n=1}^d \Theta^{(1)}_{2j_1}\ldots \Theta^{(n)}_{j_{n-1}j_n} X_{j_n}\right)^2\right] \nonumber \\
=& \sum_{(\boldsymbol{j}^1,\boldsymbol{j}^2,\boldsymbol{j}^3,\boldsymbol{j}^4)} \Cov\Big[\Theta^{(1)}_{j_0^1j_1^1}\ldots\Theta^{(n)}_{j_{n-1}^1j_n^1}X_{j_n^1}\Theta^{(1)}_{j_0^2j_1^2}\ldots\Theta^{(n)}_{j_{n-1}^2j_n^2}X_{j_n^2},\nonumber\\
&\hspace{10em}\Theta^{(1)}_{j_0^3j_1^3}\ldots\Theta^{(n)}_{j_{n-1}^3j_n^3}X_{j_n^3}\Theta^{(1)}_{j_0^4j_1^4}\ldots\Theta^{(n)}_{j_{n-1}^4j_n^4}X_{j_n^4}\Big],
\end{align}
where $(\boldsymbol{j}^1,\boldsymbol{j}^2,\boldsymbol{j}^3,\boldsymbol{j}^4)$ is a tuple of $\boldsymbol{j}^i = (j_0^i, \ldots, j_n^i)$ such that $j_0^1 = j_0^2 = 1$, $j_0^3 = j_0^4 = 2$.  Note that if $(\boldsymbol{j}^1,\boldsymbol{j}^2,\boldsymbol{j}^3,\boldsymbol{j}^4)$ is not a bad tuple as in the definition of Lemma \ref{lem:bad_tuple}, the respected covariance value in the sum in equation~\eqref{eqn:cov_with_different_coefficient} is equal to 0. By contrast, when $(\boldsymbol{j}^1,\boldsymbol{j}^2,\boldsymbol{j}^3,\boldsymbol{j}^4)$ is a bad tuple, note that if for each $\ell\in\{0,1,\ldots,n\}$, we have $j_\ell^1 = j_\ell^2$, $j_\ell^3 = j_\ell^4$ and $j_\ell^1 \neq j_\ell^3$, we have 
\begin{equation*}
    \Cov\left[\Theta^{(1)}_{j_0^1j_1^1}\ldots\Theta^{(n)}_{j_{n-1}^1j_n^1}X_{j_n^1}\Theta^{(1)}_{j_0^2j_1^2}\ldots\Theta^{(n)}_{j_{n-1}^2j_n^2}X_{j_n^2},
\Theta^{(1)}_{j_0^3j_1^3}\ldots\Theta^{(n)}_{j_{n-1}^3j_n^3}X_{j_n^3}\Theta^{(1)}_{j_0^4j_1^4}\ldots\Theta^{(n)}_{j_{n-1}^4j_n^4}X_{j_n^4}\right] = 0.
\end{equation*}
\vspace{0.5em}
\noindent
There are a total of $[d(d-1)]^n$ such tuples. Otherwise, by following the same arguments in the calculation in equation~\eqref{eqn:covariance_in_bad_case}, we get
\begin{align*}
&\Cov\left[\Theta^{(1)}_{j_0^1j_1^1}\ldots\Theta^{(n)}_{j_{n-1}^1j_n^1}X_{j_n^1}\Theta^{(1)}_{j_0^2j_1^2}\ldots\Theta^{(n)}_{j_{n-1}^2j_n^2}X_{j_n^2},
\Theta^{(1)}_{j_0^3j_1^3}\ldots\Theta^{(n)}_{j_{n-1}^3j_n^3}X_{j_n^3}\Theta^{(1)}_{j_0^4j_1^4}\ldots\Theta^{(n)}_{j_{n-1}^4j_n^4}X_{j_n^4}\right]\\
&\leq \dfrac{3^n}{d^{2n}}\max_{1\leq i\leq n}\mathbb{E}\left[X_i^4\right].
\end{align*}
Thus, using part $(ii)$ of the Lemma \ref{lem:bad_tuple} about the number of bad tuples, we obtain an upper bound of the covariance term in equation~\eqref{eqn:estimate_alpha_neural_network}
\begin{align}
\label{eqn:bound_when_different_first_coefficient}
&\Cov\left[\left(\sum_{j_1=1}^d\ldots\sum_{j_n=1}^d \Theta^{(1)}_{1j_1}\ldots \Theta^{(n)}_{j_{n-1}j_n} X_{j_n}\right)^2, \left(\sum_{j_1=1}^d\ldots\sum_{j_n=1}^d \Theta^{(1)}_{2j_1}\ldots \Theta^{(n)}_{j_{n-1}j_n} X_{j_n}\right)^2\right]\nonumber\\
&\leq\dfrac{\bigO(3^n)}{d}\max_{1\leq j \leq d} \mathbb{E}\left[X_j^4\right].
\end{align}
Putting the equations \eqref{eqn:estimate_alpha_neural_network}, \eqref{eqn:bound_when_the_same_first_coefficient}, and \eqref{eqn:bound_when_different_first_coefficient} together, we have $[A(\mu^*)]^2 \leq \bigO(3^nd)\max_{1\leq j \leq d} \mathbb{E}\left[X_j^4\right]$, which implies that
\begin{align*}
    A(\mu^*) \leq \bigO(3^{\frac{n}{2}}d^{\frac{1}{2}}).
\end{align*}
\textbf{An upper bound of $B_k(\mu^*)$:} Finally, we need to bound $B_k(\mu^{*})$ for $k\in\{1,2\}$. By applying the Cauchy-Schwartz inequality, we get $B_1(\mu^{*}) \leq B_2(\mu^*)$. Thus, it is sufficient to bound $B_2(\mu^*)$. Denote by ${\bX^*}^{\prime}=({X^*_1}^{\prime},\ldots,{X^*_d}^{\prime})$ an independent copy of $\bX^*=(X^{*}_1,\ldots,X^{*}_d)$, we consider
\begin{equation*}
    \langle \bX^*,{\bX^*}^{\prime}\rangle^2 = \left(\sum_{j=1}^d X^{*}_j{X^*_j}^{\prime}\right)^2 = \sum_{j=1}^d (X^{*}_j)^2({X^*_j}^{\prime})^2 + 2\sum_{1 \leq i < j \leq d} X^{*}_iX^{*}_j{X^*_i}^{\prime}{X^*_j}^{\prime}.
\end{equation*}
Since $\bX^*$ and ${\bX^*}^{\prime}$ are i.i.d random vectors, we have for any $1\leq j\leq d$ that
\begin{align*}
    \mathbb{E}[(X^{*}_j)^2({X^*_j}^{\prime})^2] &= (\mathbb{E}[(X^{*}_j)^2])^2 = \left(\mathbb{E}\left[\left(\sum_{j_1=1}^d\ldots\sum_{j_n=1}^d \Theta^{(1)}_{jj_1}\ldots \Theta^{(n)}_{j_{n-1}j_n} X_{j_n}\right)^2\right]\right)^2\\
    &= \left(\mathbb{E}\left[\sum_{j_1=1}^d\ldots\sum_{j_n=1}^d \left(\Theta^{(1)}_{jj_1}\right)^2\ldots \left(\Theta^{(n)}_{j_{n-1}j_n}\right)^2 X^2_{j_n}\right]\right)^2 \\
    &\leq \left(d^n \max_{1\leq j \leq d}\dfrac{1}{d^n} \mathbb{E}[X_j^2]\right)^2 = \left(\max_{1\leq j \leq d}\mathbb{E}[X_j^2]\right)^2.
\end{align*}
Meanwhile, for $1\leq i\neq j\leq d$, we get
\begin{align*}
    &\mathbb{E}[X^{*}_iX^{*}_j{X^*_i}^{\prime}{X^*_j}^{\prime}]=(\mathbb{E}[X^{*}_iX^{*}_j])^2 \\
    =& \left(\mathbb{E}\left[\left(\sum_{j_1=1}^d\ldots\sum_{j_n=1}^d \Theta^{(1)}_{ij_1}\ldots \Theta^{(n)}_{j_{n-1}j_n} X_{j_n}\right)\left(\sum_{j_1=1}^d\ldots\sum_{j_n=1}^d \Theta^{(1)}_{jj_1}\ldots \Theta^{(n)}_{j_{n-1}j_n} X_{j_n}\right)\right]\right)^2\\
    =&~ 0.
\end{align*}
Consequently, we obtain $\mathbb{E}[\langle \bX^{*}, {\bX^*}^{\prime}\rangle^2] \leq 
    d\max_{1\leq j \leq d} (\mathbb{E}[X_j^2])^2$, which leads to
\begin{equation*}
    B_2(\mu^{*}) \leq d^{\frac{1}{2}} \max_{1\leq j \leq d} \mathbb{E}[X_j^2]=\bigO(d^{\frac{1}{2}}).
\end{equation*}
In summary, we have
\begin{align*}
\label{eqn:final_calculation_for_neural_network}
    \mathfrak{m}_2(\mu^{*}) &\leq \bigO(d),\nonumber\\
    A(\mu^{*}) &\leq \bigO(3^{\frac{n}{2}}d^{\frac{1}{2}}),\nonumber\\
    B_1(\mu^{*})&\leq B_2(\mu^{*}) \leq \bigO(d^{\frac{1}{2}}).
\end{align*}
By plugging those results into equation~\eqref{eqn:xi_star_def}, we obtain
\begin{align*}
    \Xi_d(\mu^*)\leq\bigO(3^{\frac{n}{2}}d^{-\frac{1}{2}}+d^{-\frac{1}{4}}).
\end{align*}
Likewise, we also have $\Xi_d(\nu^*)\leq\bigO(3^{\frac{n}{2}}d^{-\frac{1}{2}}+d^{-\frac{1}{4}})$. Hence, we reach the conclusion of the theorem, which is
\begin{align*}
    (\Xi_d(\mu^*)+\Xi_d(\nu^*))^{\frac{1}{2}}\leq \bigO(3^{\frac{n}{4}}d^{-\frac{1}{4}}+d^{-\frac{1}{8}}).
\end{align*}
\subsection{Proof of Lemma~\ref{lm:pair_of_poly}}
\label{appendix:pair_of_poly}
\textbf{Part (i)}. Recall that an element of the set $P^m_d$ is of the form $P(x_1,\ldots,x_d):=x_1^{a_1}\ldots x_d^{a_d}$, where $a_i\in\mathbb{N}\cup\{0\}$ for all $i\in\{1,\ldots,d\}$ and $\sum_{i=1}^da_i=m$.
Therefore, the cardinality of $P^m_d$ is equal to the number of non-negative integer solutions of the following equation: 
\begin{equation*}
    a_1+\ldots+a_d = m.
\end{equation*}
Thus, we obtain that $|P^m_d|=\binom{d+m-1}{m}$.

\vspace{0.5em}
\noindent
\textbf{Part (ii)}. For any $i\in\{1,\ldots,d\}$, let $A_i$ be the set of all pairs of elements in $P_d^m$ such that both two elements in each pair contain $x_i$. Then, by using the principle of inclusion-exclusion, we have
\begin{equation}
    \label{eqn:inclu_exclu_inequality}
    \sum_{i = 1}^d |A_i| - \sum_{1\leq i<j\leq d}|A_i \cap A_j|\leq |A_d^m| \leq   \sum_{i = 1}^d |A_i|.
\end{equation}
\textbf{Compute $|A_i|$:}
It is worth noting that for any pair $(P,Q)\in A_i$, we have $P(x_1,\ldots,x_d)/x_i\in P^{m-1}_d$ and $Q(x_1,\ldots,x_d)/x_i\in P^{m-1}_d$. Since $P$ is not necessarily different from $Q$, then we have
\begin{align}
    \label{eqn:compute_Ai}
    |A_i|=|P^{m-1}_d|\times |P^{m-1}_d| = |P^{m-1}_d|^2 = \binom{d+m-2}{m-1}^2.
\end{align}
\textbf{Compute $|A_i\cap A_j|$, $i<j$:} Similarly, note that for each $(P,Q)\in A_i\cap A_j$, we have $P(x_1,\ldots,x_d)/(x_ix_j)\in P^{m-2}_d$ and $Q(x_1,\ldots,x_d)/(x_ix_j)\in P^{m-2}_d$. Then, we obtain
\begin{align}
    \label{eqn:compute_Ai_Aj}
    |A_i\cap A_j|=|P^{m-2}_d|\times |P^{m-2}_d| = |P^{m-2}_d|^2 = \binom{d+m-3}{m-2}^2.
\end{align}
Plugging the results in equations~\eqref{eqn:compute_Ai} and \eqref{eqn:compute_Ai_Aj} into equation~\eqref{eqn:inclu_exclu_inequality}, we get
\begin{equation*}
\label{eqn:bound_of_the_first_case_polynomial}
d\binom{d+m-2}{m-1}^2 - \dfrac{d(d-1)}{2}\binom{d+m-3}{m-2}^2 \leq |A_d^m| \leq d\binom{d+m-2}{m-1}^2.
\end{equation*}
Hence, we achieve the conclusion of this part. 

\vspace{0.5em}
\noindent
\textbf{Part (iii)}. The product of two elements in each pair in $B_d^m$ has the form 
\begin{equation*}
    x_1^{a_1}\ldots x_d^{a_d},
\end{equation*}
where $a_i\neq 1$ is an integer number for any $i\in\{1,\ldots,d\}$, and $a_1+\ldots + a_d = 2m$. We now count how many tuples $(a_1,\ldots,a_d)$ satisfy those constraints. 
\begin{itemize}
    \item Assume that exactly $k$ elements of $\{a_1,\ldots,a_d\}$ are greater than 1 while others equal 0. There are a total of $\binom{d}{k}$ such tuples.
    \item Next, we consider a subset of $k$ elements of $\{a_1,\ldots,a_d\}$ such that they are greater than 1 and they sum up to $2m$. Without loss of generality, we can assume that this subset is $\{a_1,\ldots,a_k\}$. Then, we have $a_i \geq 2$ for all $i \in \{1,\ldots,k\}$, and $a_1+a_2+\ldots+a_k = 2m$, or equivalently, $(a_1-2)+\ldots+(a_k-2) = 2m - 2k$. Thus, there are $\binom{2m-k-1}{2m-2k}=\binom{2m-k+1}{k-1}$ ways to choose $(a_i)_{i=1}^k$ such that $a_i \geq 2$ for all $1\leq i\leq k$ and $a_1+\ldots+a_k = 2m$. 
\end{itemize}
Thus, there are $\sum_{k = 1}^d \binom{d}{k} \binom{2m-k-1}{k-1}$ ways to choose a non-negative integer tuple $(a_1,\ldots,a_d)$ such that $a_1+\ldots+a_d = 2m$ and $a_i\neq 1$ for all $i\in\{1,\ldots,d\}$, which is a polynomial of $d$ with degree $m$. 

Subsequently, we consider a pair $(P,Q)\in B^m_d$ such that $P(x_1,\ldots,x_d)Q(x_1,\ldots,x_d)=x_1^{a_1}\ldots x_d^{a_d}$ where $(a_1,\ldots,a_d)$ is a tuple of non-negative integers satisfying that $a_i\geq 2$ for all $i\in\{1,\ldots,k\}$ ($1\leq k\leq d)$ while other $a_i$ equal 0, and $a_1+\ldots+a_d=2m$. Suppose that $P$ and $Q$ can be written as
\begin{align*}
    P(x_1,\ldots,x_d)&=x_1^{u_1}\ldots x_d^{u_d},\\
    Q(x_1,\ldots,x_d)&=x_1^{v_1}\ldots x_d^{v_d},
\end{align*}
where $u_i$ and $v_i$ are non-negative integers for all $i\in\{1,\ldots,d\}$. Then, we have
\begin{align*}
    u_i+v_i=a_i,
\end{align*}
for any $i\in\{1,\ldots,k\}$. Obviously, there are $(1+a_1)\ldots(1+a_d)$ pairs of $(u_i)_{i=1}^d$ and $(v_i)_{i=1}^d$ satisfying those constraints. As a result, we have
\begin{align*}
    |B^m_d|=(1+a_1)\ldots(1+a_d)\sum_{k=1}^d\binom{d}{k}\binom{2m-k-1}{k-1}.
\end{align*}
Note that
\begin{equation*}
    1\leq (a_1+1)\ldots(a_k+1)\leq \left(\dfrac{a_1+\ldots+a_k+k}{k}\right)^k = \left(1+\dfrac{2m}{k}\right)^k \leq e^{2m},
\end{equation*}
or equivalently,
\begin{align*}
    \sum_{k=1}^d\binom{d}{k}\binom{2m-k-1}{k-1}\leq |B^m_d|\leq e^{2m}\sum_{k=1}^d\binom{d}{k}\binom{2m-k-1}{k-1}.
\end{align*}
Hence, we reach the conclusion of this part.
\subsection{Proof of Lemma~\ref{lem:bad_tuple}}
\label{appendix:bad_tuple}
\textbf{Part (i)}. We consider bad tuples of the form $((j_0^1,j_1^1,\ldots,j_n^1)$, $(j_0^2,j_1^2,\ldots,j_n^2)$, $(j_0^3,j_1^3,\ldots,j_n^3)$, $(j_0^4,j_1^4,\ldots,j_n^4))$ such that $j_0^1 = j_0^2 = j_0^3 = j_0^4 = 1$. Among them, let $A_n$ be the number of tuples such that $j_n^1 = j_n^2 = j_n^3 = j_n^4$, and $B_n$ be the number of tuples such that there exist a way to divide the set $\{1,2,3,4\}$ into two disjoint subsets $\{p,q\}$ and $\{r,s\}$ such that $j_n^{p} = j_n^{q}$, $j_i^{r} = j_i^{s}$, and $j_n^{p} \neq j_n^{r}$. After some computations, we have the following recursive formula: $$ A_{n+1} = dA_n + dB_n, \quad B_{n+1} = 3(d^2-d)A_n + (d^2-d)B_n, $$ with $A_1 = d$, $B_1 = 3(d^2-d)$. Then, by using the induction method, we obtain that $A_n$ is a polynomial of variable $d$ of degree $2n - 1$ while $B_n$ is a polynomial of variable $d$ of degree $2n$. Thus, the number of bad tuples such that $j_0^1 = j_0^2 = j_0^3 = j_0^4 = 1$ is a polynomial of variable $d$ of degree $2n$. 

\vspace{0.5em}
\noindent
\textbf{Part (ii)}. Let $A_n$ and $B_n$ be defined as in part (i), with the difference is that $j_0^1 = j_0^2 = 1, j_0^3 = j_0^4 = 2$. We have the same recursive formulae 
\begin{equation*}
    A_{n+1} = dA_n + dB_n, \quad B_{n+1} = 3(d^2-d)A_n + (d^2-d)B_n, 
\end{equation*}
with the initial condition $A_1 = d$, $B_1 = d^2-d$. Using the induction method, we receive the desired result.
\section{Circular Defining Function}
\label{sec:circular}
In this appendix, we present the challenge of approximating the generalized sliced Wasserstein distance under the setting of a circular defining function, which is defined as follows:
\begin{definition}[Circular defining function]
Let $\bx$ and $\theta$ be vectors in $\br^d$ while $t\in\br^+$ be a positive real number. Then, the circular defining function is given by
\begin{align*}
    g_{\mathsf{circular}}(\bx,\theta):=\|\bx-t\theta\|,
\end{align*}
where $\|\cdot\|$ is the Euclidean norm. In addition, the generalized sliced Wasserstein in this case is denoted as $\mathsf{circular-}GSW$.
\end{definition}
\vspace{0.5em}
\noindent
It can be seen that the circular defining function cannot be written as the inner product of $\theta$ and some quantity dependent on $\bx$. Therefore, under this setting, we are not able to utilize the conditional central limit theorem for Gaussian projections in Theorem~\ref{theorem:clt_random_projections}, which requires to project the input measures $\mu$ and $\nu$ along the direction of $\theta$. Additionally, we are also incapable of demonstrating the relation between the $\mathsf{circular-}GSW$ distance and SW distance as in Proposition~\ref{prop:GSW_SW_relation} (when the defining function is polynomial)  and Proposition~\ref{prop:neural_GSW_SW_relation}  (when the defining function is neural network type) due to the same reason. As a consequence, the problem of finding a deterministic approximation of the $\mathsf{circular-}GSW$ distance remains open, and we believe that it is essential to develop a new technique to tackle the aforementioned issue, which is left for future work.

\bibliography{reference}
\bibliographystyle{abbrv}

\end{document}